\documentclass{article} 
\usepackage{iclr2027_conference, times}

\usepackage{amsmath,amsfonts,bm,amssymb}

\def\eqref#1{equation~\ref{#1}}

\def\1{\bm{1}}

\def\mbx{{\mathbf{x}}}

\def\mbX{{\mathbf{X}}}

\def\mcX{{\mathcal{X}}}
\def\mcY{{\mathcal{Y}}}
\def\mcH{{\mathcal{H}}}
\def\mcA{{\mathcal{A}}}

\def\mcL{{\mathcal{L}}}

\def\mcS{{\mathcal{S}}}

\DeclareMathAlphabet{\mathsfit}{\encodingdefault}{\sfdefault}{m}{sl}
\SetMathAlphabet{\mathsfit}{bold}{\encodingdefault}{\sfdefault}{bx}{n}

\def\sP{{\mathbb{P}}}

\newcommand{\E}{\mathbb{E}}

\newcommand{\VoI}{\textsc{VoI}(\mbx)}

\newcommand{\hVoI}{\widehat{\textsc{VoI}}(\mbx)}

\usepackage{hyperref}
\usepackage{url}

\usepackage[capitalize,nameinlink]{cleveref}
\usepackage{subcaption}  
\usepackage{hyperref}
\usepackage{natbib}
\usepackage{multirow}
\usepackage{booktabs}
\usepackage{float}
\usepackage{graphicx}
\usepackage{amsthm}
\theoremstyle{plain}
\usepackage{bbm}
\usepackage[most]{tcolorbox}
\usetikzlibrary{shapes.geometric}
\usepackage[dvipsnames]{xcolor}
\usepackage{cleveref}
\crefname{appendix}{Appendix}{Appendices}
\Crefname{appendix}{Appendix}{Appendices}

\usepackage{tikz}
\definecolor{hNoReq}{HTML}{CC78BC}  
\definecolor{hReq}{HTML}{DE8F05}   
\definecolor{hpcolor}{HTML}{56B4E9}
\definecolor{mpcolor}{HTML}{cc78bc}
\definecolor{mpfullcolor}{HTML}{fbafe4}

\newcommand{\TLearner}{\texttt{T-Learner}}
\newcommand{\ClasswiseRisk}{\texttt{ClassWise}}
\newcommand{\Random}{\texttt{Random}}
\newcommand{\SLearner}{\texttt{S-Learner}}
\newcommand{\Confidence}{\texttt{Confidence}}

\newcommand{\SynthB}{\texttt{SynthBin}}
\newcommand{\SynthM}{\texttt{SynthMulti}}
\newcommand{\Synth}{\texttt{Synth}}
\newcommand{\Email}{\texttt{Email}}
\newcommand{\Image}{\texttt{ImageNet-16H}}
\newcommand{\hzero}{\texttt{NoDisc}}
\newcommand{\hone}{\texttt{FullDisc}}
\newcommand{\ie}[0]{\textit{i.e.,}}

\newcommand{\mkoct}[1]{\tikz[baseline=-0.55ex]{\node[regular polygon,regular polygon sides=8,  fill=#1,draw=black,line width=0.35pt,inner sep=0pt,minimum size=1.46ex] {};}}
\newcommand{\mkcross}[1]{\tikz[baseline=-0.55ex]{\draw[fill=#1,draw=black,line width=0.35pt]
  ( 0.30ex, 0.30ex) -- ( 0.92ex, 0.30ex) -- ( 0.92ex,-0.30ex) -- ( 0.30ex,-0.30ex) --
  ( 0.30ex,-0.92ex) -- (-0.30ex,-0.92ex) -- (-0.30ex,-0.30ex) -- (-0.92ex,-0.30ex) --
  (-0.92ex, 0.30ex) -- (-0.30ex, 0.30ex) -- (-0.30ex, 0.92ex) -- ( 0.30ex, 0.92ex) -- cycle;}}

\newcommand{\mkex}[1]{\tikz[baseline=-0.55ex]{\draw[fill=#1,draw=black,line width=0.35pt,rotate=45,scale=0.84]
  ( 0.30ex, 0.30ex) -- ( 0.92ex, 0.30ex) -- ( 0.92ex,-0.30ex) -- ( 0.30ex,-0.30ex) --
  ( 0.30ex,-0.92ex) -- (-0.30ex,-0.92ex) -- (-0.30ex,-0.30ex) -- (-0.92ex,-0.30ex) --
  (-0.92ex, 0.30ex) -- (-0.30ex, 0.30ex) -- (-0.30ex, 0.92ex) -- ( 0.30ex, 0.92ex) -- cycle;}}

\newcommand{\HP}{\mkex{hpcolor}}     
\newcommand{\MP}{\mkcross{mpcolor}}     
\newcommand{\MPFull}{\mkoct{mpfullcolor}}

\title{Should I Stay or Should I Show?\\Learning to Selectively Disclose Information}

\author{Carlotta Giacchetta$^{1}$\thanks{Corresponding author.} \quad
Alessando Bogani$^{1}$  \quad
Cesare Barbera$^{2,1}$ \quad
Giovanni De Toni$^{3}$\thanks{Work done while at Fondazione Bruno Kessler (FBK).} \\
\textbf{Michele Caprio}$^{4}$ \quad
\textbf{Andrea Pugnana}$^{1,}$\thanks{Joint last authorship.} \quad
\textbf{Andrea Passerini}$^{1,\ddagger}$ \\[4pt]
$^{1}$University of Trento, Trento, Italy \quad
$^{2}$University of Pisa, Pisa, Italy \\
$^{3}$ETH AI Center \& ETH, Z\"urich, Switzerland \quad
$^{4}$University of Warwick, Coventry, UK \\[4pt]
\texttt{\{name.surname\}@unitn.it}, \quad
\texttt{cesare.barbera@phd.unipi.it}, \\
\texttt{giovanni.detoni@ai.ethz.ch}, \quad
\texttt{michele.caprio@warwick.ac.uk}
}

\newtheorem{theorem}{Theorem}
\newtheorem{proposition}{Proposition}

\iclrfinalcopy 
\begin{document}

\maketitle

\begin{abstract}

In many high-stakes settings, human decision-makers can acquire support information before making a decision. However, acquiring information is costly, and disclosure may fail to improve human decisions or may even impair them.
We tackle this problem by studying \emph{selective disclosure}, i.e., the problem of learning when to reveal support information to a human decision-maker under a budget constraint. We first show that the optimal policy is a threshold rule on the Value of Information (VoI), i.e., the expected reduction in human decision risk induced by disclosure.
Since VoI is unknown in practice, we estimate the regime-specific human risks and bound the possible degradation of the resulting plug-in policy relative to lack of disclosure, as well as its regret relative to the optimal policy.
Experiments on benchmark datasets show that selective disclosure outperforms both no disclosure and full disclosure, regardless of whether the support information is beneficial or harmful.
Two user studies show that human-AI team performance can improve when disclosure is led by our learned policy and not human-selected, although this advantage varies across tasks. A counterfactual benchmark, which replaces participants' predictions with a machine-learning prediction when disclosure occurs, suggests that these differences might depend on lower adherence to advice when the information is automatically provided rather than self-requested.
\end{abstract}

\section{Introduction}

In many decision-making tasks, humans predict uncertain outcomes with the aid of support information. Financial analysts may conduct costly research to inform investment decisions~\citep{xiong2023secret}, physicians may order invasive diagnostic tests~\citep{kasivisvanathan2018mri}, and zoologists may collect biological samples to identify endangered species~\citep{li2017applying}. Such information can improve decisions, but acquiring it is often costly~\citep{saar2009active}. Moreover, support information may also be redundant or even harmful when it introduces noise, confusion, or over-reliance on imperfect external feedback~\citep{DBLP:journals/ais/RomeoC26}. Information disclosure must therefore balance the cost of providing information against its potential effect on decision quality.

To tackle this challenge, we introduce \emph{Learning to Selectively Disclose} (LSD), a decision-theoretic framework that learns when to show additional support information to a human decision-maker within a limited disclosure budget (\cref{fig:lsd}). Rather than asking whether information is informative in general, LSD asks a more decision-relevant question: \emph{will revealing information improve the final decision?} We formalize this benefit through the conditional Value of Information (VoI), defined as the reduction in human decision risk caused by disclosure. 

More precisely, we show that the optimal policy prescribes disclosure for cases with the largest positive VoI. Because VoI is not observable, we cast its estimation as a causal inference problem: we map VoI to the negative conditional average treatment effect of information disclosure on human decision loss and we establish conditions for its identification. Then, we derive guarantees for plug-in policies relative to no disclosure and the optimal policy, and introduce a class-wise VoI estimator for classification under the $0-1$ loss. Experiments on both synthetic and real-world data show that the learned policy identifies instances in which disclosure improves human decisions. 
We further evaluate our approach with two user studies on two different tasks, showing that our learned policy either outperforms or performs on par with human-selected disclosure. A counterfactual benchmark, which replaces human predictions with Machine Learning (ML) predictions for disclosed items, suggests that lower adherence rates to the ML predictions may partly account for these results.

\textbf{Our Contributions.}
Our main contributions are:
\begin{enumerate}
\item We present selective disclosure as a two-regime decision problem
and show that the optimal budget-constrained policy discloses information for
instances with the largest positive conditional Value of Information (VoI);
\item We formulate selective disclosure as a causal intervention on the decision-maker's information, characterizing the conditional VoI as the causal reduction in human decision risk induced by disclosure. Then, we establish the assumptions under which VoI is identifiable;
\item We derive performance guarantees for plug-in disclosure policies in terms of their VoI estimation error, both relative to the no-disclosure baseline and relative to the optimal policy under the same budget constraint;
\item For classification under the $0$-$1$ loss, we introduce a class-wise estimator of the regime-specific risks and demonstrate its effectiveness on synthetic and real-world datasets;
\item We run two user studies to compare human-AI team performance in classification tasks under policy-selected and human-selected disclosure.

\end{enumerate}%
\begin{figure}[t]
    \centering
    \includegraphics[width=.8\linewidth]{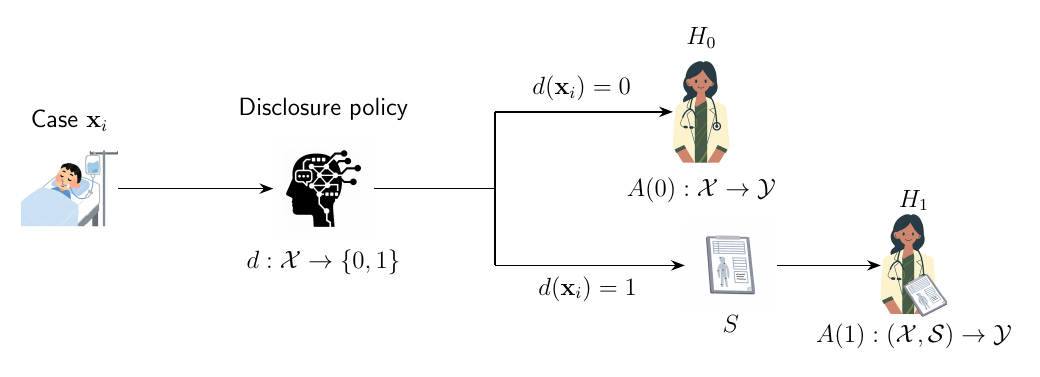}
    \caption{\textbf{LSD pipeline}. For each case $\mbx_i$, a disclosure policy $d : \mcX \to \{0,1\}$ decides when to disclose support information $S$ to a human expert $H$. If no support information is disclosed ($d(\mbx_i) = 0$), a human expert makes a decision using only baseline features. When support information is disclosed ($d(\mbx_i) = 1$), a human expert uses both baseline and support information.}
    \label{fig:lsd}
\end{figure}

\section{Preliminaries} 
\label{sec:background}

We study a decision-making setting in which a human observes fixed baseline information in order to perform an action. Before selecting the action, the human may additionally receive support information to help its decision. More formally, let $\mcX\subseteq\mathbb{R}^m$ be the baseline information space and let $\mcY= \{1,\ldots,|\mcY|\}$ be a finite target space. A case consists of baseline information $\mbX\in\mcX$, a target $Y\in\mcY$, and support information $S\in\mcS$, where $S$ may, e.g., be an ML model prediction. Although the realization of $S$ may differ across cases, we treat it as either disclosed in full or withheld.
A binary variable $D\in\{0,1\}$ denotes the disclosure regime. 
We define the action space $\mcA$ and, throughout the paper, we set $\mcA=\mcY$, so that the human action is a prediction of the target label. 
Let $\mcH$ denote a population of human decision-makers and let $\pi$ be a distribution over $\mcH$. For a decision-maker $H\sim\pi$, we denote by $A_H(0)\in \mcA$ and $A_H(1)\in\mcA$ the actions that $H$ would take under no disclosure and disclosure, respectively. These actions may vary across decision-makers even for the same case.
We further assume access to a dataset  $\mathcal{T} =\{(\mbx_i,y_i,a_{Hi}(0),a_{Hi}(1))\}_{i=1}^n,$ sampled from some unknown $P$ over $\mcX\times\mcY\times \mcA\times\mcA$, where $a_{Hi}(0)$ and $a_{Hi}(1)$ are decisions collected for case $i$ under the two disclosure regimes. In particular, each case is assigned to (at least) two distinct decision-makers $H_{0},H_{1}\sim\pi$, with $H_{0}$ operating under $D=0$ and $H_{1}$ operating under $D=1$. Notably, the dataset contains observations from both regimes for each case, but it does not contain both potential actions for the same human-case episode. Thus, our objective is to learn a disclosure policy for a decision-maker drawn from the population $\pi$, rather than a policy for a fixed human.

\section{Learning to Selectively Disclose}
\label{sec:methodology}
The objective of LSD is to determine when a decision-maker should be given access to support information under a budget constraint.
We frame the problem from a decision-theoretic perspective.  Given a non-negative loss function \(\ell:\mcA\times\mcY\to\mathbb{R}^+\) we define the regime-specific conditional risks:
\begin{equation}
\label{eq:risks}
    r_D(\mbx)= 
\E\left[
\ell\bigl(A_H(D),Y\bigr)
\mid \mbX=\mbx
\right],
\qquad D\in{0,1},
\end{equation}
where the expectation is taken over the target $Y$, the decision-maker $H\sim\pi$, and any randomness in the human response, conditional on $\mbX=\mbx$. Thus, $r_D(\mbx)$ measures the expected loss under disclosure regime $D$ for an \emph{average} decision-maker drawn from $\pi$, rather than on any particular individual. Smaller values correspond to more reliable decisions, while larger values indicate greater expected error.
We define the \emph{conditional Value of Information} (VoI) as the expected reduction in decision risk induced by disclosure, namely:
\begin{equation}
\label{eq:voi}
\VoI =  r_0(\mbx) - r_1(\mbx)
\end{equation}
with positive values indicating that revealing $S$ improves decision quality. 

Let $\mathcal{D}$ be the set of measurable disclosure policies $d:\mcX\to\{0,1\}$, where $d(\mbx)=1$ indicates that the support information is disclosed. We denote the population disclosure risk with $R(d) = \E_\mbx[(1-d(\mbx))r_0(\mbx) + d(\mbx)r_1(\mbx)]$, i.e.\ the expected loss incurred when each instance is judged under the regime that $d$ selects for it; in particular, $R(0)$ and $R(1)$ are the no- and full-disclosure baselines, obtained by never and always revealing $S$.
Typically, access to more information is costly, thus we assume there is a budget level $B \in [0, 1]$ that limits the expected disclosure rate. The Learning to Selectively Disclose problem can be formulated as follows:
\begin{equation}
    \min_{d \in \mathcal D} \quad R(d) \qquad \text{s.t.} \quad \mathbb{E}_{\mbx}     [d(\mbx)] \le B
\end{equation}

In the next subsections, we characterize the optimal disclosure policy for the LSD problem, formalize the causal estimation of VoI, provide theoretical guarantees for the corresponding plug-in approaches and then showcase how to estimate them when $\ell$ is the 0-1 loss, a common measure in human-AI collaboration literature~\citep{DBLP:conf/aaai/RuggieriP25}.

\subsection{Optimal Disclosure Policy}
\label{subsec:ODP}
A first question to address is what an optimal policy for LSD looks like. We show that the optimal policy consists of a threshold rule over the VoI, as stated in the following Theorem:
\begin{theorem}[Optimal Disclosure Policy]
\label{thm:optimal_disclosure}
Assume that the cumulative distribution of $\VoI$ is continuous. Let $q_{1-B}$ denote the $(1-B)$-quantile of  $\VoI$ and let $\lambda^* = \max\{0,\, q_{1-B}\}$. For any budget $B \in [0,1]$, the optimal disclosure policy is the threshold rule:
\begin{equation}
d^*(\mbx) =
\begin{cases}
1 & \text{if} \quad \VoI \geq \lambda^*, \\
0 & \text{otherwise}.
\end{cases}
\end{equation}
\end{theorem}
\begin{proof}
    We provide the proof in~\cref{app:proof_th1}. 
\end{proof}
\Cref{thm:optimal_disclosure} shows that selective disclosure admits an intuitive interpretation: support information is revealed only when it is expected to \emph{benefit} the decision-maker, i.e., when the expected risk reduction is positive and exceeds a budget-dependent threshold. 
However, $\VoI$ is not directly observable, since a decision-maker either receives $S$ or does not. As a result, the two risks defining $\VoI$ cannot be jointly observed for the same decision instance. 
Estimating $\VoI$ is therefore naturally  a causal inference problem. %
In the following section, we formalize the corresponding causal framework and state the assumptions under which $\VoI$ is identifiable from observed data.

\subsection{A causal inference interpretation}
\label{subsec:causal}
Since $\VoI$ is never observed for the same human decision-maker, a natural way to handle its estimation is resorting to the Neyman-Rubin potential outcomes framework~\citep{rubin1974estimating}.
However, our framework differs from the canonical treatment-allocation setting in an important aspect. In standard policy learning~\citep{kitagawa2018should}, a treatment is assigned to a unit and directly affects that unit's outcome. In LSD, instead, the intervention acts on a decision-maker: disclosure changes the information available to a human who must make a decision about a separate case. The case outcome $Y$ is not affected by disclosure; what may change is the human action $A_H(D)$ and, consequently, the decision loss $\ell(A_H(D),Y)$. 
Thus, $\ell(A_H(0),Y)$ and $\ell(A_H(1),Y)$ are the two potential outcomes of interest, and the causal effect of our interest is not the effect of an intervention on the case itself, but the effect of disclosing support information on the quality of a human decision. 
More precisely, for a fixed decision-maker $h$, we define the human-specific conditional effect of disclosure on decision loss as  $\tau_h(\mbx) = \mathbb{E}\!\left[ \ell\!\left(A_h(1),Y\right) - \ell\!\left(A_h(0),Y\right) \mid X=\mbx,H=h \right]$. Because our objective is to learn a policy for a future decision-maker drawn from a target population $\pi$, we consider the population-average effect 
$\tau_\pi(\mbx) = \mathbb{E}_{H\sim\pi}\!\left[\tau_H(\mbx)\right]$ and note that 
$-\tau_\pi(\mbx)=\VoI{}.$
Hence, $\VoI{}$ measures the causal effect of disclosing information on the quality of a decision made about a case with covariates $\mbx$, averaged over decision-makers from $\pi$.

Following causal inference practice~\citep{nogueira2022methods}, we now investigate under which assumptions $\VoI{}$ can be retrieved from empirical data. We assume\footnote{For the sake of space, we present these assumptions in detail in~\cref{app:causal}.}: $(A1)$ \textbf{Unconfoundedness}, \ie{} conditional on baseline covariates $\mbX$, the disclosure decision is as good as random with respect to the potential regime-specific risks~\citep{rosenbaum1983central}; 
$(A2)$ \textbf{Positivity}, \ie{} for all covariate profiles in the population, we observe both disclosure regimes with non-zero probability~\citep{10.1214/aos/1176344064}; 
$(A3)$ \textbf{Consistency}, \ie{} if a decision episode is observed under a specific regime, then the observed loss is equal to the corresponding potential loss~\citep{cole2009consistency}; 
$(A4)$ \textbf{Stable Unit Treatment Value Assumption (SUTVA)}, \ie{} for each decision episode, the potential action of the human decision-maker depends only on the information disclosed to that decision-maker in that episode~\citep{rubin1981estimation} and 
$(A5)$ \textbf{Exchangeability of Human Decision-Makers}, \ie{} the decision-makers are drawn independently of the covariates and of the regime. 
Under assumptions \textit{A1-A5}, we can estimate our causal estimand from data:

\begin{proposition}[Identifiability of the Value of Information]
\label{prop:voi_identifiability}
Under Assumptions A1--A5, the regime-specific potential risks are
identifiable from observed data. Let
$
\ell^{\mathrm{obs}}
\bigl(A_{H_D}(D),Y\bigr)
=
(1-D)\ell\bigl(A_{H_0}(0),Y\bigr)
+
D\ell\bigl(A_{H_1}(1),Y\bigr)
$
denote the loss observed in a decision episode, where \(H_D\) is the
decision-maker assigned under regime \(D\). Then, for
\(d\in\{0,1\}\), the Value of Information
\[
\VoI
=
r_0(\mbx)-r_1(\mbx)
=
\mathbb E\!\left[
\ell\bigl(A_H(0),Y\bigr)
-
\ell\bigl(A_H(1),Y\bigr)
\mid \mbX=\mbx
\right]
\]
is identifiable as
\[
\VoI
=
\mathbb E\!\left[
\ell^{\mathrm{obs}}\bigl(A_H(D),Y\bigr)
\mid \mbX=\mbx,D=0
\right]
-
\mathbb E\!\left[
\ell^{\mathrm{obs}}\bigl(A_H(D),Y\bigr)
\mid \mbX=\mbx,D=1
\right].
\]
\end{proposition}

\begin{proof}
    We provide the proof in~\cref{proof:voi_identifiability}
\end{proof}

\Cref{prop:voi_identifiability} reduces the $\VoI$ to two identifiable, regime-specific conditional risks. Estimating it becomes a standard CATE estimation problem, for which any estimator from the causal inference literature can be used. The simplest one (T-learner~\cite{kunzel2019metalearners}) fits a risk model per regime, $\hat r_0$ and $\hat r_1$, and sets $\hVoI = \hat r_0(\mbx) - \hat r_1(\mbx)$. Still, accurate prediction of $\VoI$, is not sufficient for effective disclosure: the policy depends on the sign~\citep{DBLP:conf/nips/FrauenMSSF25} and ranking~\citep{DBLP:journals/corr/abs-2602-03517} of $\hVoI$, and estimation errors can cause the policy to disclose information where it is harmful. In the next subsection, we quantify how such errors affect the learned policy's performance.

\subsection{Guarantees for the plug-in policy}
\label{subsec:guarantees}

Since $\VoI$ is identifiable from data collected under the two regimes (\cref{prop:voi_identifiability}), one can threshold an estimate $\hVoI$ in its place, 
obtaining what we call a \emph{plug-in} policy. 
Let us denote with $\hat d(\mbx) = \mathbbm{1}\{\hVoI \geq \hat\lambda\}$ the plug-in disclosure policy, with threshold $\hat\lambda \ge 0$ for a given budget $B$, with $\hat B = \E_\mbx[\hat d(\mbx)]$ its realized disclosure rate and with $\varepsilon(\mbx) = \hVoI - \VoI$ the estimation error of $\hVoI$. Thresholding $\hVoI$ raises two questions: whether the resulting policy can be worse than no-disclosure and how far it can fall short of the optimal policy.

First, we address whether deploying the plug-in policy can achieve higher risk than simply never disclosing. We stress that this is possible in principle because disclosure can be harmful on average ($\E_\mbx[\VoI] < 0$) and a policy that mis-ranks instances can spend its budget where disclosure is harmful.
In the following, we quantify how large this degradation can be:
\begin{theorem}[No-disclosure degradation]
\label{thm:degradation} Given a non-negative loss function \(\ell:\mcA\times\mcY\to\mathbb{R}^+\), for any estimator $\hVoI$ such that $\|\varepsilon\|_{L^2(p)} < \infty$ and any $\hat\lambda \geq 0$, the following holds:
\begin{equation}
  R(\hat d) - R(0) \;\le\; \sqrt{\hat B}\,\|\varepsilon\|_{L^2(p)} \;-\; \hat\lambda \hat B .
  \label{eq:safety}
\end{equation}
\end{theorem}

\begin{proof}
    We provide the proof in~\cref{app:proof_noDisc}. 
\end{proof}

The bound is governed by the $\VoI$ estimation error and by the region on which the policy intervenes: for bounded $\|\varepsilon\|_{L^2(p)}$, it converges to zero as the realized disclosure rate $\hat B$ tends to zero but it is positive whenever $\|\varepsilon\|_{L^2(p)} > \hat\lambda\sqrt{\hat B}$. Hence, the plug-in policy can perform worse than never disclosing, but only by a bounded amount (see~\cref{app-par:degradation}).

We now present our \emph{Oracle regret bound}. This addresses how much is lost, relative to the optimal policy, by thresholding an estimate rather than the true $\VoI$. Here we additionally require the plug-in rule to use the threshold that~\cref{thm:optimal_disclosure} prescribes for $\hVoI$, i.e., the rule that is optimal according to its own estimate of the risks: 
\begin{theorem}[Oracle regret]
\label{thm:regret}
Given a non-negative loss function \(\ell:\mcA\times\mcY\to\mathbb{R}^+\), let $d^\ast$ be the optimal policy of~\cref{thm:optimal_disclosure} and $B^\ast = \E_\mbx[d^\ast(\mbx)]$. Assume that the cumulative distribution of $\hVoI$ is continuous, let $\hat q_{1-B}$ be its $(1-B)$-quantile and $\hat\lambda = \max\{0,\,\hat q_{1-B}\}$. Then, it holds that
\begin{equation}
  R(\hat d) - R(d^\ast) \;\le\; \sqrt{B^\ast + \hat B}\,\|\varepsilon\|_{L^2(p)} .
  \label{eq:oracle_regret}
\end{equation}
\end{theorem}
\begin{proof}
    We provide the proof in~\cref{app:proof_guarantees}.
\end{proof}
The Oracle-regret bound of~\cref{eq:oracle_regret} shows that the plug-in rule is $O(\|\varepsilon\|_{L^2(p)})$-competitive with an oracle knowing the true risks. Such a bound scales with an upper bound on the mass of the disagreement region, $\E_\mbx[|d^\ast(\mbx)-\hat d(\mbx)|]\leq B^\ast+\hat B$. The bound can therefore be loose when the two policies agree away from the threshold: estimation errors on instances that both policies disclose, or both withhold, do not affect their relative performance.

\subsection{Estimating the Value of Information for 0-1 loss} 
\label{subsec:methodology-estimating} 
As shown in~\cref{thm:optimal_disclosure}, learning disclosure policies requires estimating the regime-specific risks \(r_0(\mbx)\) and \(r_1(\mbx)\) and then thresholding their difference. 
In what follows, we focus on the 0-1 loss, which is widely used in human-AI collaboration to model the error incurred by the human decision maker~\citep{DBLP:conf/nips/MadrasPZ18,DBLP:conf/aaai/RuggieriP25}. 
The $0-1$ loss is defined as:
\begin{equation}
\label{eq:zero-one}
    \ell(A(D), Y) = \mathbbm{1}\{A(D)\neq Y\},
\end{equation}
Under this loss choice, the expected risk becomes:
\begin{equation}
\label{eq:0-1risk}
    r_D(\mbx)= \E\!\left[\mathbbm{1}\{A(D)\neq Y\} \right], \qquad D \in \{0,1\}
\end{equation}
Thus, estimating ~\cref{eq:0-1risk} reduces to a standard classification task, where the goal is to predict whether the human makes mistakes, as a T-learner would do with a single scalar risk model per regime~\citep{kunzel2019metalearners}. However, humans do not err uniformly across classes: some classes are systematically confused with others, and disclosure alters these error patterns differently for each class. A scalar risk model condenses all these patterns into a single error probability, blurring the class-specific effects on which the VoI depends (see \cref{app:q1}). We therefore decompose the regime-specific risk class-wise, relying on the following factorization of the $0-1$ risk:

\begin{proposition}[Class-wise decomposition]
    \label{prop:class-wise}
    Let us consider $\ell(A(D),Y)=\mathbbm{1}\{A(D)\neq Y\}$. Then $r_{D}(\mbx)$ can be written as:
    \begin{equation}
         r_D(\mbx)=\sum_{y\in\mcY}
    \underbrace{\sP(Y = y \mid \mbX=\mbx)}_{p_y}\,
    \underbrace{\sP(A(D)\neq y \mid \mbX=\mbx, Y = y)}_{q_{D,y}}.
    \end{equation}
\end{proposition}
\begin{proof}
    We provide the proof in~\cref{proof:Class-wise}.
\end{proof}
This decomposition suggests a separate estimation strategy for both $p_y$ and $q_{D,y}$.
More precisely, we approximate $p_y$ using a probabilistic classifier
\(
f : \mcX \to \Delta^{|\mcY|},
\)
and $q_{D,y}$ via a multi-head architecture with $|\mathcal{Y}|$ heads and a shared backbone, where any $y$-th head is a binary classifier
\(
g_{D,y}: \mathcal X \rightarrow [0,1]
\).
The resulting risk estimator is then \(
\widehat r_D(\mbx)
=
\sum_{y\in\mcY}
 f_y(\mbx) \cdot  g_{D,y}(\mbx),
\)
where $f_y$ refers to the $y-$th output of the probabilistic classifier.
The estimated VoI is obtained by plugging $\widehat r_D(\mbx)$ into~\cref{eq:voi}.
We refer to this estimator as \ClasswiseRisk{}, which specializes the T-learner to a structured class-wise risk model.

\section{Experimental Evaluation}
\label{sec:exp}

In this section\footnote{Code is available at the \href{https://anonymous.4open.science/r/submissionICLR2027-4163}{anomymous repo.} 
}, we address the following research questions: 
\begin{itemize}
  \item[\textbf{Q1}] 
  Does \ClasswiseRisk{} learn effective policies?

  \item[\textbf{Q2}] 
  Does the \ClasswiseRisk{} policy improve human-AI team performance?
  \item[\textbf{Q3}] 
  How do participants interact with the \ClasswiseRisk{} policy?
\end{itemize}

\subsection{Experimental settings}
\label{subsec:exp_setting}

\textbf{Datasets.} We consider both synthetic and real data.
We generate $10000$ synthetic data samples using the standard \texttt{make\_classification} function of \texttt{sklearn} library, considering both a binary ($|\mathcal{Y}| = 2$) and a multiclass task ($|\mathcal{Y}| = 4$). We emulate selective disclosure by treating a subset of features as baseline information and the remaining ones as support information. Then, we simulate human decisions in the two regimes with logistic regression models. 
We also employ two different real datasets: \Email{}~\citep{DBLP:journals/corr/abs-2511-21448} and \Image{}~\citep{steyvers2022bayesian}.  
(i) \Email{} contains $1000$ emails that can be classified as either legitimate or fraudulent. For this dataset, we take human predictions from~\citep{bogani2026conceptbottleneckmodelseffective}, 
who provide for each email a human prediction made either without any assistance ($D = 0$) or with the support of an ML model ($D = 1$). 
(ii) \Image{} contains $1200$ noisy images that must be assigned to one of $16$ categories. Human decisions for this dataset come from two behavioural experiments: the ones with no assistance come from~\citep{steyvers2022bayesian} ($D = 0$), while predictions with support information come from~\citep{DBLP:conf/icml/StraitouriR24} ($D=1$), where participants make their decision after observing a set of conformal predictions ($\alpha=.05$)\footnote{A detailed analysis of the effect of the miscoverage level $\alpha$ is provided in~\cref{app:alpha}.}. We provide further details in~\cref{app:impl_details}.

\textbf{Methods.} For \textbf{Q1}, we evaluate our approach (\ClasswiseRisk{}) against the following baselines: $(i)$ an uncertainty-based policy (\Confidence{}) that thresholds $u(\mbx) = 1 - \max_{y} f(\mbx)_y$ instead of $\hVoI$, at the same $1-B$ quantile~\citep{DBLP:journals/dmlr/PugnanaPDR24}; $(ii)$ a policy that randomly assigns disclosure for each instance (\Random{}); $(iii)$ a full disclosure baseline (\hone); $(iv)$ and a no-disclosure baseline (\hzero). 
For \textbf{Q2} and \textbf{Q3}, we compare \ClasswiseRisk{} against human participants who decide by themselves, item by item, whether to request the support information.

\textbf{Metrics and Evaluation.} 
For both \textbf{Q1} and \textbf{Q2} we evaluate human decision performance using test-set accuracy ($Acc$). For \textbf{Q3} we also measure participants' advice adherence ($Adh$), \ie{} the proportion of trials in which support information is available and the participant's classification matches such information (this amounts to selecting the ML-suggested label for \Email{} and any class included in the prediction set for \Image{}). We evaluate \textit{Adherence} only on policy-selected items, ensuring that comparisons between MP and HP are based on the same set of items.
User studies data are analyzed with logistic mixed-effects models with random intercepts for participants and items (see \cref{app:q2} and \cref{app:q3} for full analyses details).

\textbf{Q1 setup.} For \textbf{Q1}, $(i)$ we perform a $70/10/20$ split of the data into training, calibration, and test sets; $(ii)$ we carve out a validation split from the training data ($10\%$) and select hyperparameters by grid search\footnote{Implementation details are provided in~\cref{app:impl_details} (\Synth{}~\cref{app:impl_details_synth}, \Email{}~\cref{app:impl_details_email}, \Image{}~\cref{app:impl_details_cp}).}; 
$(iii)$ we calibrate the disclosure threshold as the $1-B$ quantile of $\hVoI$ on the calibration set for each $B\in\{.1,.2,.3,.4,.5,.6,.7,.8,.9\}$ and $(iv)$ we report the human accuracy on the test set, averaged over five seeds that vary model initialization\footnote{We do not resample the split across seeds, as the user studies require a single, fixed test set.}.

\textbf{User studies.}
To answer \textbf{Q2} and \textbf{Q3}, we run two user studies (545 total participants) in which each participant classifies $20$ items sampled from the test set of one of the two real datasets, \ie{} \Email{} or \Image{}. 
We manipulate two between-subjects variables, with two conditions each: \emph{(i) Support} is either \emph{Human-Policy} (HP), where participants decide themselves when to request the support information, (i.e., the ML's model prediction, and specifically the predicted label for emails and the conformal prediction set at $\alpha=.05$ for images), 
or \emph{Machine-Policy} (MP), where our \ClasswiseRisk{} policy decides when to provide it; \emph{(ii) Budget} is either \emph{Low} (LB) or \emph{High} (HB), making the support information available for at most $6$ or $14$ of the $20$ items ($30\%$ and $70\%$). Participants in HP are told they need not exhaust their budget and are incentivized to request information only when needed, so that their policy is broadly comparable to the learned one. 
The $20$-item samples are stratified so that $5/9$ emails and $6/14$ images come from the items the policy would select for disclosure under LB/HB, matching its realized disclosure rates in the original test sets, and all four conditions draw from the same pool\footnote{Realized rates need not match $B$, since the threshold is calibrated on a held-out split and the policy never discloses when $\hVoI\le 0$; they also vary across runs, more so on \Email{} (see~\cref{app:hum_exp}).}.
Similarly to Scenario~1 in~\citep{DBLP:conf/aistats/PalombaPAR25}, the ML model's prediction is retrospectively available for every case. We therefore report, as an exploratory analysis, a counterfactual benchmark ({MPCFT}), obtained by $(i)$ replacing MP participants' responses with the underlying ML model prediction whenever information is disclosed and $(ii)$ retaining their observed responses otherwise. Thus, {MPCFT} represents the accuracy that would have been achieved had the ML prediction been enforced as the final decision whenever it was disclosed. See \cref{app:hum_exp} for further details on the user studies.

\subsection{Experimental Results}
\label{subsec:results}
\textbf{Q1: \ClasswiseRisk{} is effective on all datasets and outperforms simple baselines.}~\cref{fig:results} reports the results for our experiments on all four datasets.
For \Synth{} we can see that \ClasswiseRisk{} dominates the \Confidence{} baseline and \Random{} at every budget level.
Interestingly, \ClasswiseRisk{} already exceeds \hone{} accuracy at $B=.30$ in the \SynthB{} ($Acc\approx0.83$) and at $B=.50$ in the \SynthM{} ($Acc\approx0.63$), suggesting that disclosing information for only a small fraction of instances suffices to recover the performance of full disclosure.

For \Email{}, disclosure is harmful on average: \hzero{} outperforms \hone{} ($Acc\approx.79$ vs $Acc\approx.76$). \ClasswiseRisk{} is the best-performing policy for all $B$, peaking at $B=.30$ ($Acc\approx.81$).  
Beyond this point, the policy increasingly discloses on instances where it wrongly estimates support to be beneficial, so accuracy decreases and plateaus slightly below \hzero{}, which however remains within its confidence band. As anticipated by \cref{thm:degradation}, the plug-in policy can thus perform worse than never disclosing, although only by a bounded amount.
Unlike \ClasswiseRisk{}, the \Confidence{} baseline never abstains: its accuracy decreases monotonically with the budget, never exceeds \hzero{}, and collapses into \hone{} at $B=1$. The reason is that uncertainty measures how hard the \emph{case} is, not how much disclosure \emph{helps}: there is no mechanism to abstain when disclosure hurts.

For \Image{}, we observe that on average disclosing support information is helpful: \hone{} achieves $Acc\approx.85$ while \hzero{} $Acc\approx .77$. Interestingly, \ClasswiseRisk{} surpasses the full disclosure baseline already at $B=.70$ and achieves the highest accuracy of $\approx.86$ at $B=.90$. Moreover, we observe that \ClasswiseRisk{} is always different from \Random{}, suggesting that the policy can learn when support information is useful. \Confidence{} is a stronger baseline here, but it only matches \ClasswiseRisk{} at $B \geq .80$, where nearly every instance is disclosed and little is left to choose. Whenever the budget is binding, \ClasswiseRisk{} is always ahead.
\begin{figure}[t]
    \centering
    \includegraphics[width=\linewidth]{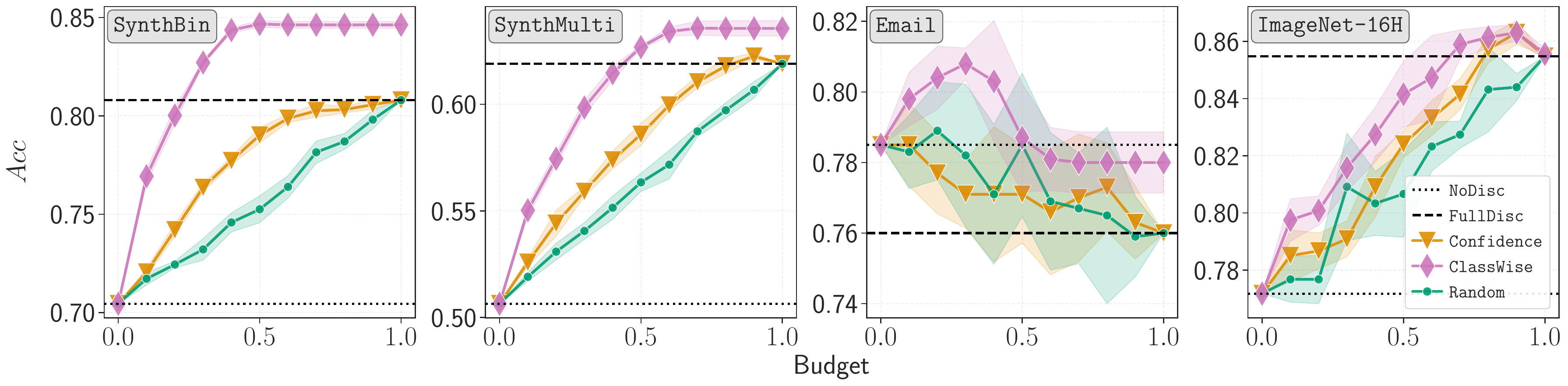}
    \caption{Budget-accuracy curves for the four considered datasets. We report average mean values and $95\%$ confidence intervals over five seeds.}
    \label{fig:results}
\end{figure}

Moreover, we ablate the risk estimator considering two variants that do not exploit the classwise decomposition of~\cref{prop:class-wise}, i.e., $(a)$ an approach that directly estimates both $r_0$ and $r_1$ separately (\TLearner{}) and $(b)$ a joint estimator for both $r_0$ and $r_1$ (\SLearner{}). 
\ClasswiseRisk{} outperforms both variants at every budget on \Synth{} data, and attains the highest peak accuracy on the real datasets. The gap is largest on \Email{}, where disclosure is harmful: \SLearner{} shrinks $\hVoI$ towards zero, estimates a positive VoI almost everywhere and thus never abstains, collapsing to \hone{} at high budgets, whereas \ClasswiseRisk{} withholds support information on a larger fraction of the instances than \TLearner{}. Full ablation and budget analyses are in \cref{app:q1}.

\textbf{Q2: The learned policy matches or outperforms self-selection.} 
We report the two user-study results on human accuracy ($Acc$) in~\cref{fig:q2}.
When considering \textit{Support}, in the \Image{} task accuracy is significantly higher in MP than in HP (\MP{}: $0.80 \pm 0.11$ vs \HP{}: $0.77 \pm 0.13$, $p = .025$). Instead, in the \Email{} task MP and HP are not significantly different (\MP{}: $0.82 \pm 0.12$ vs \HP{}: $0.83 \pm 0.13$, $p = .762$). %
When considering \textit{Budget}, in the \Image{} task accuracy is significantly higher in HB than in LB ($0.83 \pm 0.10$ vs $0.74 \pm 0.13, p < .001$), whereas the two conditions do not significantly differ in the \Email{} task ($0.84 \pm 0.12$ vs $0.81 \pm 0.12, p = .056$).
The advice-adherence results reported next may help account for why MP outperforms HP in the \Image{} but not in the \Email{} task. \Cref{app:q2} reports full results, plus analyses of advice correctness, human confidence, and alignment between \ClasswiseRisk{} and human disclosure.

\begin{figure}[t]
\centering
\begin{subfigure}[t]{0.49\linewidth}
    \centering
    \includegraphics[width=\linewidth]{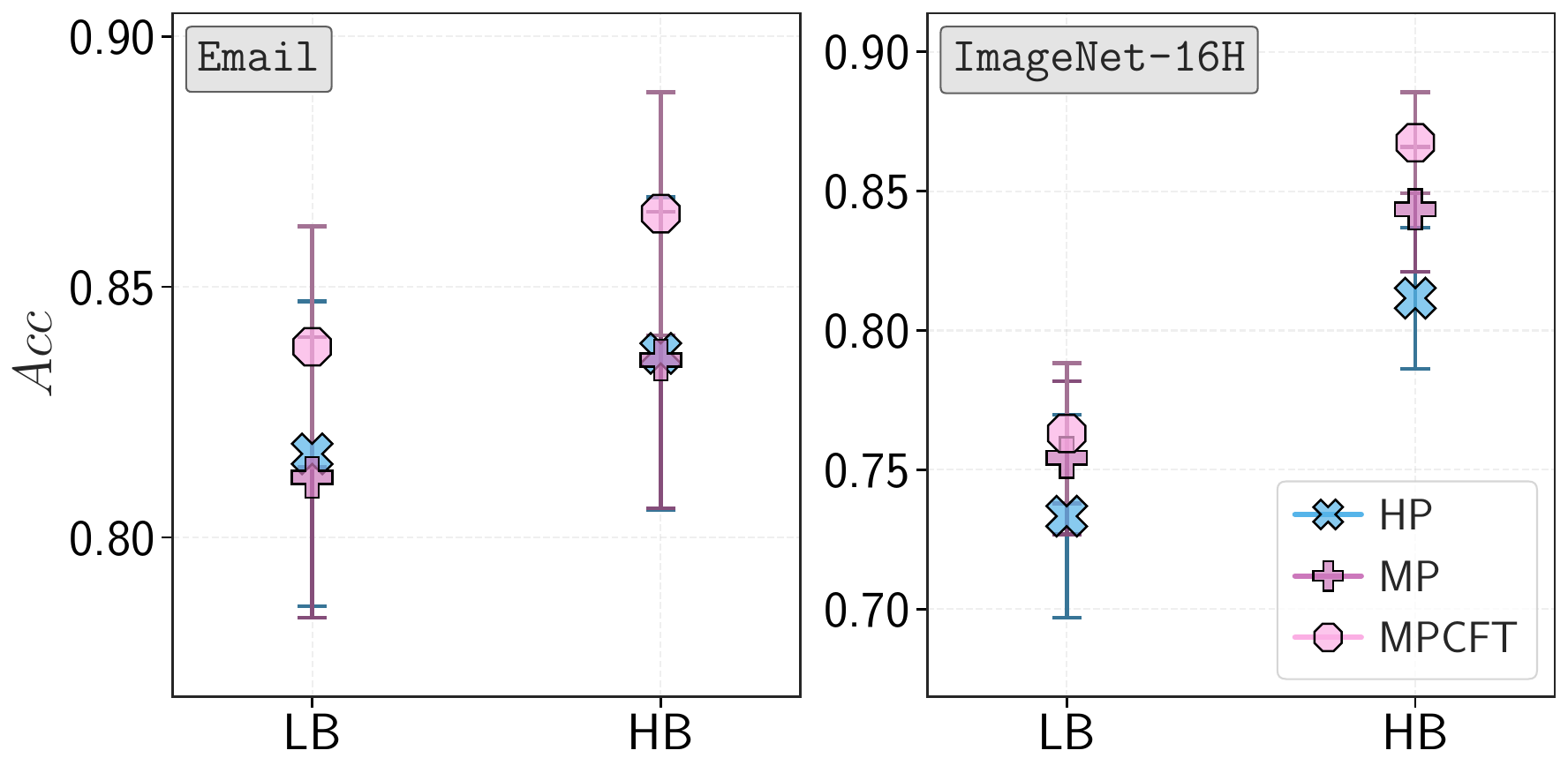}
    \caption{Accuracy}
    \label{fig:q2}
\end{subfigure}
\hfill
\begin{subfigure}[t]{0.49\linewidth}
    \centering
    \includegraphics[width=\linewidth]{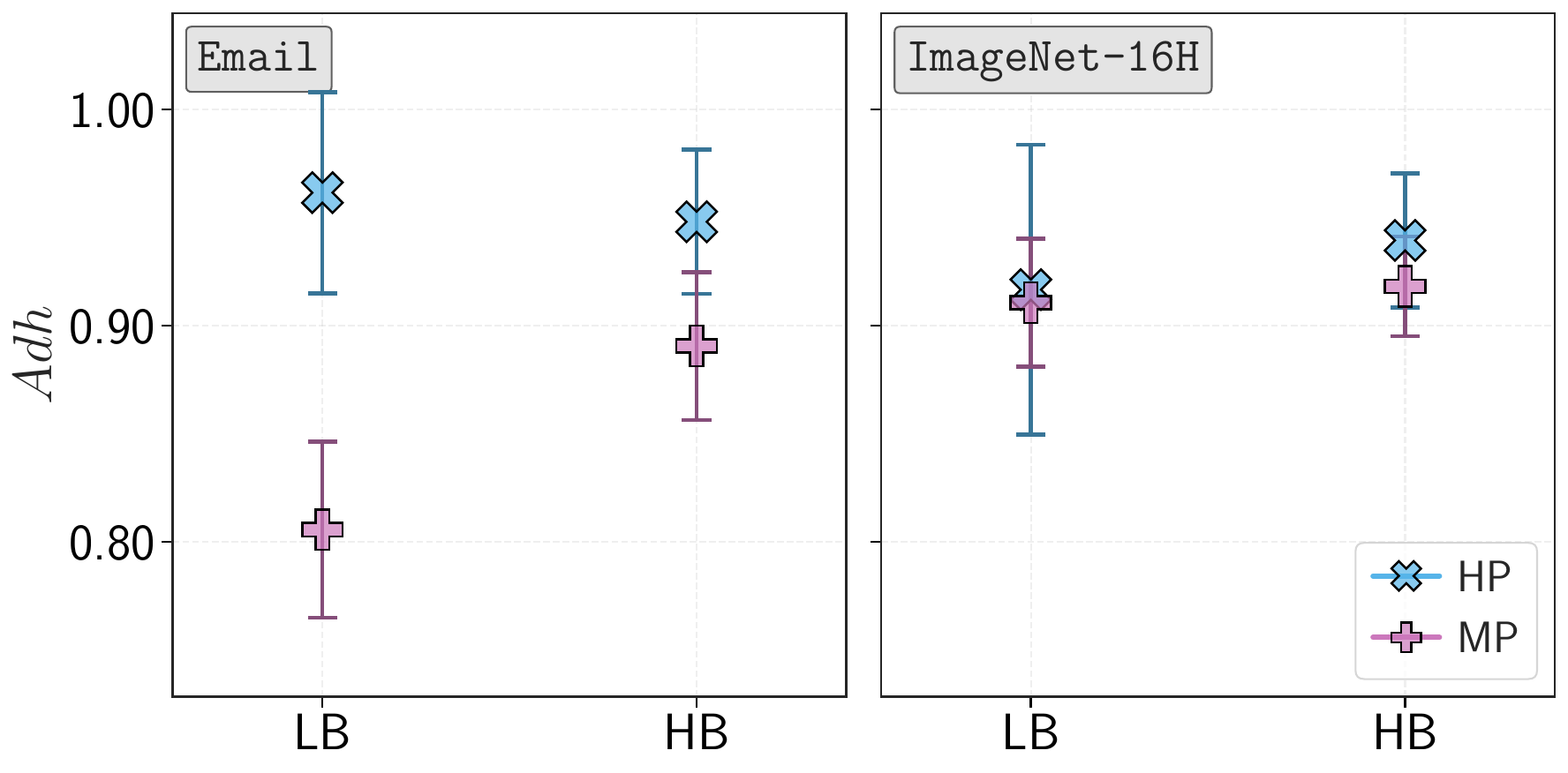}
    \caption{ Adherence}
    \label{fig:q3}
\end{subfigure}
\caption{User studies results (\textbf{Q2} and \textbf{Q3}). (a) Participants' accuracy ($Acc$) and (b) adherence ($Adh$) on the AI advice, by \textit{Budget} (LB/HB) and \textit{Support} (HP/MP/MPCFT) condition, on \Email{} (left) and \Image{} (right). Error bars are $95\%$ confidence intervals across participants.}
\label{fig:human-exp}
\end{figure}

\textbf{Q3: The learned policy is more beneficial whenever humans follow the advice.} 
We report results on advice-adherence ($Adh$) in~\cref{fig:q3}. When considering \textit{Support}, in the \Image{} task adherence is not significantly different between MP and HP (\MP{}: $0.91 \pm 0.11$ vs \HP{}: $0.93 \pm 0.18$, $p = .084$), while in the \Email{} task it is significantly lower in MP than in HP (\MP{}: $0.85 \pm 0.16$ vs \HP{}: $0.96 \pm 0.15$, $p < .001$).
When considering \textit{Budget}, in both studies adherence does not differ significantly between HB and LB (\Image{}: $p = .260$; \Email{}: $p = .636$), not even in interaction with \textit{Support} (\Image{}: $p = .392$; \Email{}: $p = .126$).

Accuracy in MPCFT (\cref{fig:q2}) suggests that lower advice-adherence may attenuate the benefits of selective disclosure. In the \Image{} task, where adherence under MP is already high, accuracy in MPCFT is descriptively but not significantly higher than in MP (\MPFull{}: $0.82 \pm 0.10$ vs. \MP{}: $0.80 \pm 0.11$; $p = .058$), while it is significantly higher than in HP (\HP{}: $0.77 \pm 0.13$; $p < .001$). In \Email{}, where adherence under MP is lower, MPCFT accuracy is significantly higher than in MP (\MPFull{}: $0.85 \pm 0.10$ vs. \MP{}: $0.82 \pm 0.12$; $p = .004$) and descriptively, but not significantly, higher than in HP (\HP{}: $0.83 \pm 0.13$; $p = .196$). This cross-task pattern suggests that lower adherence may partly contribute to the absence of an MP advantage in the \Email{} task, as would be expected given that the disclosed ML advice is correct on most trials. Additional results are reported in~\cref{app:q3}.

Overall, these results suggest that deploying selective disclosure requires attending not only to which information is valuable, but also to how decision-makers receive information they did not request. Indeed, the pattern we observe is consistent with psychological evidence suggesting that self-produced outcomes are evaluated more favorably than comparable outcomes produced by external sources~\citep{botti2023choice, enisman2021choice}.

\section{Related Work}
\label{sec:related}
Our paper builds upon related work on \textit{human-AI decision-making} and \textit{policy learning under budget constraints}. We provide an extended overview of these and related topics in \cref{app:related}.

\textbf{Human-AI Decision-Making.} Learning to Defer~\citep{DBLP:conf/nips/MadrasPZ18} and follow-up work~\citep{DBLP:conf/icml/MozannarS20,DBLP:conf/icml/VermaN22} allocate each instance to the model or to the human, and causal reasoning has so far been used to evaluate such systems~\citep{DBLP:conf/aistats/PalombaPAR25}; in LSD the human always decides, and the policy allocates \emph{information}. Other work designs the form of support, e.g., conformal prediction sets~\citep{DBLP:conf/nips/ToniOTSR24} or explanations~\citep{DBLP:conf/iui/SchemmerKBBS23}, while~\citet{noorani2026humanai} let the AI refine a human expert's prediction set, recovering labels the human missed without degrading correct human judgments.~\citet{DBLP:conf/ijcai/Noti023} and~\citet{10.1145/3544548.3581058} are closely related to our work, as they show support when the AI is predicted to outperform the human. Similarly,~\citet{DBLP:conf/aaai/BhattCCKKWT25} learn personalized disclosure policies online. We differ by casting selective disclosure as a causal inference problem: disclosure is an intervention on the decision-maker's information, and the VoI is never observed under both regimes. To the best of our knowledge, LSD is the first to establish when VoI is identifiable from data collected under the two regimes (\cref{subsec:causal}) and to show that VoI is the required estimand to build budget-constrained disclosure policies with guarantees (\cref{thm:degradation,thm:regret}).

\textbf{Policy Learning.} \Cref{thm:optimal_disclosure} is the disclosure analogue of optimal policy learning~\citep{kitagawa2018should} under a capacity constraint, where the optimal rule thresholds the conditional average treatment effect at the larger of zero and a budget-dependent quantile~\citep{BHATTACHARYA2012168,DBLP:journals/ijdsa/Cerulli26}. The difference is where the treatment acts: disclosure is applied to the decision-maker, while the outcome is the correctness of a decision about a case, so SUTVA is stated over decision episodes and identification requires decision-makers to be exchangeable across regimes.

\section{Conclusion}
\label{sec:conclusions} 
In this work, we introduced Learning to Selectively Disclose (LSD), a decision-theoretic framework for deciding when to provide support information to a human decision-maker under a budget constraint.
We showed that the optimal policy thresholds the Value of Information, revealing information only where the expected benefit is positive and exceeds a budget-dependent threshold. Since the VoI is not directly observable, we gave conditions under which it is identifiable from data collected under the two regimes. We then studied plug-in approaches, for which we bounded the estimation error relative to no disclosure and the regret relative to the optimal policy.
For classification tasks under the $0-1$ loss, we further decomposed the human risk class-wise and used this decomposition to build a VoI estimator, which proved effective on both synthetic and real benchmarks. Two user studies further indicate that \ClasswiseRisk{}-selected disclosure can improve human-AI team accuracy relative to human-selected disclosure, although its effectiveness varies across tasks. This variability may partly reflect how participants respond to advice: in some contexts, adherence may be lower for automatically provided than for comparable self-requested information, and a counterfactual benchmark replacing participants' predictions with ML ones when disclosed suggests that greater adherence would have yielded higher overall accuracy under \ClasswiseRisk{}-selected disclosure.

\textbf{Limitations and Future Works.} Our framework treats the support information $S$ as an indivisible block: the policy decides \emph{whether} to disclose, never \emph{what} to disclose. Studying which parts of the information to reveal is left for future work. 
Also, our user studies do not identify which task characteristics reduce participants’ advice-adherence nor establish a causal relationship between adherence and policy effectiveness, as they were not designed to address these questions. Future work will investigate these aspects with dedicated user studies. 

\subsection*{AI use statement}

In this work, we used generative AI tools for code support and polishing the writing.
Regarding the former, the code was manually reviewed by two authors.
Regarding the latter, we have used generative AI tools to provide feedback on the presentation's quality, with the goal of avoiding concerns that are due solely to inaccuracies in our writing. ChatGPT Astra Pro was used to assist with proofreading and to improve the clarity and presentation of the manuscript.
We take responsibility for the final content of this work, including text, claims or artifacts produced with the aid of generative AI.

\subsection*{Ethics statement}

The research protocol of our user studies was determined as posing no risk to participants’ well-being or rights and therefore as not requiring full ethical review by the Research Ethics Committee of [institution omitted to preserve anonymity] ([protocol number omitted to preserve anonymity]).

\subsection*{Reproducibility statement}

The main text describes the problem formulation, the assumptions for identifying the VoI, the \ClasswiseRisk{} estimator, the baselines, the datasets, the evaluation protocol (data splits, threshold calibration, budget grid, and seeds), and the design and analysis of the user studies. Complete proofs are provided in the appendix, together with additional experimental results and details on the user studies design. The implementation and experiment scripts used to produce our results are made available through an anonymized repository, which also contains the data collected in the user studies and the corresponding statistical analysis scripts.

\subsubsection*{Acknowledgments}

The authors acknowledge the CINECA award under the ISCRA initiative for the availability of high-performance computing resources and support. 
This work was funded by the European Union. The views and opinions expressed are however those of the author(s) only and do not necessarily reflect those of the European Union, the European Health and Digital Executive Agency (HaDEA) or the European Research Executive Agency. Neither the European Union nor the granting authority can be held responsible for them. Grant Agreement no. 101120763 - TANGO.
Michele Caprio gratefully acknowledges support from the Prob\_AI Hub and the London Mathematical Society.

\bibliography{iclr2027_conference}
\bibliographystyle{iclr2027_conference}
\appendix

\section{Proofs}

\setcounter{theorem}{0}
\setcounter{proposition}{0}

\subsection{Proof of Optimal Disclosure Policy}
\label[appendix]{app:proof_th1}

For completeness, we restate here \cref{thm:optimal_disclosure} from \cref{subsec:ODP}, and then provide its proof.
\begin{theorem}[Optimal Disclosure Policy]
Assume that the cumulative distribution of $\VoI$ is continuous. Let $q_{1-B}$ denote the $(1-B)$-quantile of  $\VoI$ and let $\lambda^* = \max\{0,\, q_{1-B}\}$. For any budget $B \in [0,1]$, the optimal disclosure policy is the threshold rule:
\begin{equation}
d^*(\mbx) =
\begin{cases}
1 & \text{if} \quad \VoI \geq \lambda^*, \\
0 & \text{otherwise}.
\end{cases}
\end{equation}
\end{theorem}

\begin{proof}
We begin by first expanding the objective,
\[
  R(d) = \E\big[(1-d(\mbx))r_0(\mbx) + d(\mbx) r_1(\mbx)\big]
  = \E[r_0(\mbx)] - \E\big[d(\mbx)\,\VoI\big],
\]
and since $\E[r_0(\mbx)]$ does not depend on $d$, the original minimization problem is equivalent to
\begin{equation}
  \max_{d:\,\mcX\to[0,1]} \ \E_\mbx\big[d(\mbx)\,\VoI\big]
  \quad\text{s.t.}\quad \E_\mbx[d(\mbx)] \le B .
  \label{eq:primal}
\end{equation}

Introducing the multiplier $\lambda \ge 0$ for the budget constraint, the Lagrangian is
\[
  \mcL(d,\lambda)
  = \E\big[d(\mbx)\,\VoI\big] - \lambda\big(\E[d(\mbx)] - B\big)
  = \lambda B + \E\big[d(\mbx)\,(\VoI - \lambda)\big].
\]
For any fixed $\lambda \ge 0$,  the maximisation is point-wise in $\mbx$ and the integrand is linear in $d(\mbx) \in [0,1]$. Therefore,  any maximiser satisfies
\begin{equation}
  d_\lambda(\mbx) =
  \begin{cases}
    1, & \VoI > \lambda,\\
    0, & \VoI < \lambda,
  \end{cases}
  \qquad d_\lambda(\mbx) \in [0,1] \text{ arbitrary on } \{\VoI = \lambda\}.
  \label{eq:pointwise}
\end{equation}
It remains to select $\lambda^\ast$ satisfying \emph{complementary slackness}, $\lambda^\ast\big(\E[d_{\lambda^\ast}(\mbx)] - B\big) = 0$. This allows for both binding and non-binding budget scenarios, leading to two mutually exclusive cases.
\begin{itemize}
    \item[\emph{(i)}] \emph{Non-binding budget. }Suppose $\sP(\VoI > 0) < B$, so that $q_{1-B} \le 0$ and hence $\lambda^\ast = 0$. Then $\lambda^\ast = 0$ is admissible: the policy $d^\ast(\mbx) = \mathbf{1}\{\VoI > 0\}$ satisfies $\E[d^\ast(\mbx)] = \sP(\VoI > 0) < B$, so the constraint is slack and complementary slackness holds with $\lambda^\ast = 0$. By the continuity assumption at $\lambda^\ast = 0$ we have $\sP(\VoI = 0) = 0$, so $d^\ast(\mbx)$ agrees $p$-almost everywhere with $\mathbf{1}\{\VoI \ge 0\}$ and is in particular feasible. Note that in this case no feasible policy can spend the full budget profitably: increasing $\E[d(\mbx)]$ beyond $\sP(\VoI>0)$ requires disclosing on $\{\VoI \le 0\}$, which cannot increase the objective. In particular the equation $\E[d_\lambda] = B$ admits \emph{no} solution.

    \item[\emph{(ii)}] \emph{Binding budget.} Suppose $\sP(\VoI > 0) \ge B$; then $F_{\VoI}(0) \le 1-B$, so $q_{1-B} \ge 0$ and $\lambda^\ast = q_{1-B}$. By the continuity assumption at $\lambda^\ast$, $\sP(\VoI = \lambda^\ast) = 0$ and $F_{\VoI}(\lambda^\ast) = 1-B$, whence $\sP(\VoI > \lambda^\ast) = B$. The policy $d_{\lambda^\ast} = \mathbf{1}\{\VoI \ge \lambda^\ast\}$ therefore satisfies $\E[d_{\lambda^\ast}] = B$, and complementary slackness holds with $\lambda^\ast \ge 0$. \\

\end{itemize}
Combining the two cases, the optimal multiplier is
\begin{equation}
  \lambda^\ast = \max\big\{0,\; F_{\VoI}^{-1}(1-B)\big\},
  \label{eq:optimal-lambda}
\end{equation}

and, since $\sP(\VoI = \lambda^\ast) = 0$ by the continuity assumption, any maximiser of \cref{eq:primal} agrees $p$-almost everywhere with
\begin{equation}
  d^\ast(\mbx) = \mathbf{1}\{\VoI \geq \lambda^\ast\},
  \label{eq:optimal-policy}
\end{equation}
which is the statement of \cref{thm:optimal_disclosure}.
\end{proof}

\subsection{Proof of Identifiability of the Value of Information}
\label[appendix]{proof:voi_identifiability}

For completeness, we restate here \cref{prop:voi_identifiability} from \cref{subsec:causal}, and then provide its proof.

\begin{proposition}[Identifiability of the Value of Information]
Let $H \sim \pi$ denote a decision-maker drawn from the pool $\mcH$, let
$A_{H}(0)$ and $A_{H}(1)$ denote the potential actions and $\ell\bigl(A_{H}(D),Y\bigr)$
the corresponding potential losses for $D \in \{0,1\}$. Under assumptions
A1--A5, the regime-specific potential risks are identifiable from observed data. In particular, the Value of Information
\[
\VoI
=
r_0(\mbx)-r_1(\mbx)
=
\mathbb E\!\left[\ell\bigl(A_{H}(0),Y\bigr)-\ell\bigl(A_{H}(1),Y\bigr)\mid \mbX=\mbx\right]
\]
is identifiable from observed data as
\[
\VoI
=
\mathbb E\!\left[
\ell^{\mathrm{obs}}\bigl(A_{H}(D),Y\bigr)
\mid \mbX=\mbx, D=0
\right]
-
\mathbb E\!\left[
\ell^{\mathrm{obs}}\bigl(A_{H}(D),Y\bigr)
\mid \mbX=\mbx, D=1
\right]
\]
\end{proposition}

\begin{proof}
Fix $\mbx \in \mcX$ and a regime $d \in \{0,1\}$. To distinguish the random regime $D$ from the value it takes, throughout this proof we write $r_d(\mbx)$ for the risk $r_D(\mbx)$ of \cref{eq:risks} evaluated at $D = d$. Recall from \cref{sec:background} that, under regime $d$, the decision on a case is taken by the decision-maker $H_d \sim \pi$ assigned to that regime. By SUTVA (A4), the action of $H_d$ depends only on the information disclosed to it in that episode, so its potential action $A_{H_d}(d)$ is well defined. By consistency (A3), on the event $\{D = d\}$ the observed loss equals the corresponding potential loss, $\ell^{\mathrm{obs}}\bigl(A_H(D), Y\bigr) = \ell\bigl(A_{H_d}(d), Y\bigr)$. By overlap (A2), $\sP(D = d \mid \mbX = \mbx) > 0$, so conditioning on $\{\mbX = \mbx, D = d\}$ is well defined and
\[
\E\!\left[\ell^{\mathrm{obs}}\bigl(A_H(D), Y\bigr) \mid \mbX = \mbx, D = d\right]
=
\E\!\left[\ell\bigl(A_{H_d}(d), Y\bigr) \mid \mbX = \mbx, D = d\right].
\]
By exchangeability of decision-makers (A5), $H_d$ is drawn from $\pi$ independently of the covariates, and of the regime. The decision-maker observed under regime $d$ is therefore distributed as a generic $H \sim \pi$, and
\[
\E\!\left[\ell\bigl(A_{H_d}(d), Y\bigr) \mid \mbX = \mbx, D = d\right]
=
\E_{H \sim \pi}\!\left[\ell\bigl(A_H(d), Y\bigr) \mid \mbX = \mbx, D = d\right].
\]
By unconfoundedness (A1), the potential losses are independent of the regime given $\mbX$, so
\[
\E_{H \sim \pi}\!\left[\ell\bigl(A_H(d), Y\bigr) \mid \mbX = \mbx, D = d\right]
=
\E_{H \sim \pi}\!\left[\ell\bigl(A_H(d), Y\bigr) \mid \mbX = \mbx\right]
=
r_d(\mbx),
\]
where the last equality is the definition of the regime-specific risk in \cref{eq:risks}. Hence both $r_0(\mbx)$ and $r_1(\mbx)$ are identified by the observed conditional expectations, and their difference identifies $\VoI = r_0(\mbx) - r_1(\mbx)$.
\end{proof}

\subsection{Proof of the no-disclosure degradation bound}
\label[appendix]{app:proof_noDisc}

For completeness, we restate here \cref{thm:degradation} from \cref{subsec:guarantees}, and then provide its proof. Recall that $\hat d(\mbx) = \mathbbm{1}\{\hVoI \geq \hat\lambda\}$ is the plug-in policy, $\hat B = \E_\mbx[\hat d(\mbx)]$ its realized disclosure rate, and $\varepsilon(\mbx) = \hVoI - \VoI$ the estimation error of the VoI.

\begin{theorem}No-disclosure degradation]
Given a non-negative loss function \(\ell:\mcA\times\mcY\to\mathbb{R}^+\), for any estimator $\hVoI$ such that $\|\varepsilon\|_{L^2(p)} < \infty$ and any $\hat\lambda \geq 0$, the following holds:
\begin{equation*}
  R(\hat d) - R(0) \;\le\; \sqrt{\hat B}\,\|\varepsilon\|_{L^2(p)} \;-\; \hat\lambda \hat B .
\end{equation*}
\end{theorem}

\begin{proof}
Recall from the proof of \cref{thm:optimal_disclosure} that
\begin{equation}
  R(d) = \E_\mbx\big[(1-d(\mbx))\,r_0(\mbx) + d(\mbx)\,r_1(\mbx)\big]
       = \E_\mbx[r_0(\mbx)] - \E_\mbx\big[d(\mbx)\,\VoI\big].
  \label{eq:risk-voi}
\end{equation}
Since $R(0) = \E_\mbx[r_0(\mbx)]$, we have $R(\hat d) - R(0) = -\E_\mbx[\hat d(\mbx)\,\VoI]$. Writing $\VoI = \hVoI - \varepsilon(\mbx)$,
\begin{equation}
  R(\hat d) - R(0)
  = \underbrace{\E_\mbx[\hat d(\mbx)\,\varepsilon(\mbx)]}_{(\mathrm{I})}
  - \underbrace{\E_\mbx[\hat d(\mbx)\,\hVoI]}_{(\mathrm{II})} .
  \label{eq:proof-split}
\end{equation}
For $(\mathrm{I})$, we apply Cauchy--Schwarz to $\hat d$ and $\hat d\,\varepsilon$, using $\hat d(\mbx)^2 = \hat d(\mbx)$ since $\hat d$ is an indicator:
\[
  \E_\mbx[\hat d(\mbx)\,\varepsilon(\mbx)]
  \;\le\; \big(\E_\mbx[\hat d(\mbx)]\big)^{1/2}
          \big(\E_\mbx[\hat d(\mbx)\,\varepsilon(\mbx)^2]\big)^{1/2}
  \;\le\; \sqrt{\hat B}\,\|\varepsilon\|_{L^2(p)} .
\]
For $(\mathrm{II})$, on $\{\hat d(\mbx) = 1\}$ we have $\hVoI \geq \hat\lambda \geq 0$, so $\hat d(\mbx)\,\hVoI \geq \hat\lambda\,\hat d(\mbx)$ for every $\mbx$, whence $(\mathrm{II}) \geq \hat\lambda\,\hat B$. Substituting both bounds into \cref{eq:proof-split} gives the claim.
\end{proof}

The two terms have opposite signs and opposite meanings. Term $(\mathrm{I})$ is the price of estimation error, and it is paid \emph{only where the policy discloses}: the factor $\hat d(\mbx)$ zeroes out the contribution of every instance the policy leaves untouched, so that only the error on the disclosed instances, $\E_\mbx[\hat d(\mbx)\,\varepsilon(\mbx)^2]$, enters the bound, which therefore vanishes as the disclosure rate $\hat B$ shrinks. Term $(\mathrm{II})$ is instead the gain the policy \emph{believes} it is making, and enters with a negative sign: since disclosure is triggered only above the threshold, every disclosed instance contributes an estimated improvement of at least $\hat\lambda$, which offsets part of the estimation error.

\paragraph{When can the plug-in policy degrade?}
\label[appendix]{app-par:degradation}
The bound certifies no degradation only when $\|\varepsilon\|_{L^2(p)} \le \hat\lambda\sqrt{\hat B}$, i.e., when the estimation error is small relative to the margin by which disclosed instances clear the threshold; otherwise, its right-hand side is positive and no-degradation is not guaranteed. 
When it does occur, degradation has a precise source. By~\cref{eq:risk-voi}, $R(\hat d) - R(0) = -\E_\mbx[\hat d(\mbx)\,\hVoI]$, so the plug-in policy can be worse than never disclosing only if it discloses on instances where $\VoI < 0$. On these instances $\hVoI \ge \hat\lambda \ge 0$, hence $\varepsilon(\mbx) > \hat\lambda$: degradation requires the estimator to overestimate the VoI by more than the threshold on part of the disclosed region.
This is most likely when disclosure is harmful on a large fraction of instances and the budget is not binding. In this case $\hat\lambda = 0$, term $(\mathrm{II})$ provides no margin, and any instance with $\VoI < 0$ but $\hVoI \ge 0$ is disclosed. This is the regime we observe on \texttt{Email}, where \texttt{ClassWise} discloses on at most $\approx 53\%$ of the instances (\cref{fig:cp_freq_budget}): for $B \ge .70$ the threshold is zero and accuracy falls slightly below \texttt{NoDisc} (\cref{fig:results}, \cref{appfig:results_ablations}).

\subsection{Proof of the oracle regret bound}
\label[appendix]{app:proof_guarantees}

For completeness, we restate here \cref{thm:regret} from \cref{subsec:guarantees}, and then provide its proof. Recall that $\hat d(\mbx) = \mathbbm{1}\{\hVoI \geq \hat\lambda\}$ is the plug-in policy, $\hat B = \E_\mbx[\hat d(\mbx)]$ its realized disclosure rate, and $\varepsilon(\mbx) = \hVoI - \VoI$ the estimation error of the VoI.

\begin{theorem}[Oracle regret]
Given a non-negative loss function \(\ell:\mcA\times\mcY\to\mathbb{R}^+\), let $d^\ast$ be the optimal policy of~\cref{thm:optimal_disclosure} and $B^\ast = \E_\mbx[d^\ast(\mbx)]$. Assume that the cumulative distribution of $\hVoI$ is continuous, let $\hat q_{1-B}$ denote its $(1-B)$-quantile and let $\hat\lambda = \max\{0,\,\hat q_{1-B}\}$. Then
\begin{equation*}
  R(\hat d) - R(d^\ast) \;\le\; \sqrt{B^\ast + \hat B}\,\|\varepsilon\|_{L^2(p)} .
\end{equation*}
\end{theorem}

\begin{proof}
By \cref{eq:risk-voi}, and writing $\VoI = \hVoI - \varepsilon(\mbx)$,
\begin{align}
  R(\hat d) - R(d^\ast)
    &= \E_\mbx\big[(d^\ast(\mbx) - \hat d(\mbx))\,\VoI\big] \nonumber \\
    &= \underbrace{\E_\mbx\big[(d^\ast(\mbx) - \hat d(\mbx))\,\hVoI\big]}_{(\mathrm{III})}
     - \E_\mbx\big[(d^\ast(\mbx) - \hat d(\mbx))\,\varepsilon(\mbx)\big].
  \label{eq:regret_decomposition}
\end{align}

Term $(\mathrm{III})$ is non-positive. The proof of~\cref{thm:optimal_disclosure}, applied with $\hVoI$ in place of $\VoI$, shows that the threshold rule $\mathbbm{1}\{\hVoI \geq \max\{0,\,\hat q_{1-B}\}\}$ maximizes $\E_\mbx[d(\mbx)\,\hVoI]$ over all policies with $\E_\mbx[d(\mbx)] \leq B$; by assumption, $\hat d$ is that rule. Since $d^\ast$ also satisfies the budget constraint, as $B^\ast \leq B$ by~\cref{thm:optimal_disclosure}, we get $\E_\mbx[d^\ast(\mbx)\,\hVoI] \leq \E_\mbx[\hat d(\mbx)\,\hVoI]$.

Let $\Delta(\mbx) = |d^\ast(\mbx) - \hat d(\mbx)|$, the indicator of the region where the two policies disagree. Since $d^\ast(\mbx) - \hat d(\mbx) \in \{-1, 0, 1\}$, we have $\Delta(\mbx)^2 = \Delta(\mbx)$, and Cauchy--Schwarz applied to $\Delta$ and $\Delta\,|\varepsilon|$ gives
\begin{align*}
  R(\hat d) - R(d^\ast)
  &\le \E_\mbx\big[\Delta(\mbx)\,|\varepsilon(\mbx)|\big]
   \le \big(\E_\mbx[\Delta(\mbx)]\big)^{1/2}
       \big(\E_\mbx[\Delta(\mbx)\,\varepsilon(\mbx)^2]\big)^{1/2} \\
  &\le \big(\E_\mbx[\Delta(\mbx)]\big)^{1/2}\,\|\varepsilon\|_{L^2(p)} .
\end{align*}
Finally, since $d^\ast$ and $\hat d$ take values in $\{0,1\}$, $\Delta(\mbx) \leq d^\ast(\mbx) + \hat d(\mbx)$, so that $\E_\mbx[\Delta(\mbx)] \leq B^\ast + \hat B$, which gives the claim.
\end{proof}

The intermediate bound shows that the regret depends only on the region where the two policies disagree: both its mass, $\E_\mbx[\Delta(\mbx)]$, and the estimation error on it, $\E_\mbx[\Delta(\mbx)\,\varepsilon(\mbx)^2]$, enter the bound, whereas errors on instances that both policies disclose, or both withhold, do not contribute. The final bound replaces the mass of the disagreement region with the upper bound $B^\ast + \hat B$, which is loose when the two policies largely agree.

\subsection{Proof of class-wise decomposition}

For completeness, we restate here \cref{prop:class-wise} from \cref{subsec:methodology-estimating}, and then provide its proof.

\begin{proposition}[Class-wise decomposition]
    Let us consider $\ell(A(D),Y)=\mathbbm{1}\{A(D)\neq Y\}$. Then $r_{D}(\mbx)$ can be written as:
    \begin{equation}
         r_D(\mbx)=\sum_{y\in\mcY}
    \underbrace{\sP(Y = y \mid \mbX=\mbx)}_{p_y}\,
    \underbrace{\sP(A(D)\neq y \mid \mbX=\mbx, Y = y)}_{q_{D,y}}.
    \end{equation}
\end{proposition}

\begin{proof}
\label[appendix]{proof:Class-wise}
Since the expectation of an indicator is the probability of the corresponding event, the risk under the 0-1 loss is
\[
  r_D(\mbx) = \E\big[\mathbbm{1}\{A(D) \neq Y\} \mid \mbX = \mbx\big] = \sP(A(D) \neq Y \mid \mbX = \mbx).
\]
By the law of total probability over the values of $Y$,
\[
  \sP(A(D) \neq Y \mid \mbX = \mbx) = \sum_{y \in \mcY} \sP(A(D) \neq y,\, Y = y \mid \mbX = \mbx),
\]
where on the event $\{Y = y\}$ the condition $A(D) \neq Y$ becomes $A(D) \neq y$. Finally, by the chain rule of probability, each term factorizes as
\[
  \sP(A(D) \neq y,\, Y = y \mid \mbX = \mbx) = \sP(Y = y \mid \mbX = \mbx)\,\sP(A(D) \neq y \mid \mbX = \mbx,\, Y = y),
\]
which gives the claim.
\end{proof}

\section{A causal inference interpretation}
\label[appendix]{app:causal}

Here we detail the causal interpretation provided in \cref{subsec:causal}. Throughout, we use $D$ for the random disclosure regime and $d \in \{0,1\}$ for its values, and $H \sim \pi$ for a random decision-maker and $h \in \mcH$ for a fixed one. The disclosure variable $D$ plays the role of a binary treatment on the decision-maker: for a decision-maker $H$, we consider the \emph{potential} actions $A_H(0)$ and $A_H(1)$, corresponding to the decisions under the two regimes, and the associated potential losses $\ell(A_H(0), Y)$ and $\ell(A_H(1), Y)$, which are the potential outcomes of interest.

For a fixed decision-maker $h \in \mcH$, we define the human-specific risks
\[
  r_d(\mbx, h) = \E\!\left[\ell\bigl(A_h(d), Y\bigr) \mid \mbX = \mbx, H = h\right], \qquad d \in \{0,1\},
\]
so that the human-specific effect of disclosure introduced in \cref{subsec:causal} is $\tau_h(\mbx) = r_1(\mbx, h) - r_0(\mbx, h)$. This is the effect of disclosing $S$ on the decisions of $h$, rather than a treatment effect on the case itself. Under the exchangeability of decision-makers (A5 below), the decision-maker is drawn independently of the case, and the regime-specific risks of \cref{eq:risks} are the population averages $r_d(\mbx) = \E_{H \sim \pi}[r_d(\mbx, H)]$. Hence $\tau_\pi(\mbx) = \E_{H \sim \pi}[\tau_H(\mbx)] = r_1(\mbx) - r_0(\mbx)$, and $\VoI = -\tau_\pi(\mbx)$.

We now state the assumptions under which this identification is valid.

\begin{enumerate}
    \item[A1] \textbf{Unconfoundedness.} Conditional on the baseline covariates $\mbX$, the disclosure regime is as good as random with respect to the potential losses~\citep{rosenbaum1983central}:
    \[
      \bigl\{\ell(A_H(0), Y),\; \ell(A_H(1), Y)\bigr\}
      \perp\!\!\!\perp D \mid \mbX .
    \]
    Intuitively, all possible sources of self-selection into disclosure (and non-disclosure) are captured by the observable covariates.

    \item[A2] \textbf{Overlap (positivity).} For all covariate profiles in the population, both disclosure regimes are observed with non-zero probability~\citep{10.1214/aos/1176344064}. Let $e(\mbx) = \sP(D = 1 \mid \mbX = \mbx)$ denote the disclosure propensity. We assume there exists some $\eta > 0$ such that
    \[
      \eta < e(\mbx) < 1 - \eta \quad \forall \mbx \in \mcX .
    \]
    This guarantees that we can compare the risks under disclosure and non-disclosure at each covariate level and then aggregate these comparisons.

    \item[A3] \textbf{Consistency.} If a decision episode is observed under regime $D = d$, the observed loss equals the corresponding potential loss~\citep{cole2009consistency}:
    \[
      \ell^{\mathrm{obs}}\bigl(A_H(D), Y\bigr) = \ell\bigl(A_H(d), Y\bigr) \quad \text{on } \{D = d\}.
    \]

    \item[A4] \textbf{Stable Unit Treatment Value Assumption (SUTVA).} For each decision episode, the potential action of the decision-maker under regime $d \in \{0,1\}$ depends only on the information disclosed to that decision-maker in that episode~\citep{rubin1981estimation}.
    \item[A5] \textbf{Exchangeability of decision-makers.}
For each disclosure regime $d \in \{0,1\}$, the decision-maker
assigned to a decision episode is drawn from the same target population
$\pi$, independently of the case covariates and the disclosure regime.
Formally,
\[
    H_d \mid \mbX = \mbx, D = d \sim \pi,
    \qquad
    \text{for all } \quad \mbx \in \mcX \quad
    \text{ and } \quad d \in \{0,1\}.
\]
Accordingly, the regime-specific conditional risk is the expected loss
of a decision-maker drawn from $\pi$:
\[
    r_d(\mbx)
    =
    \E_{H \sim \pi}
    \left[
        \ell\bigl(A_H(d), Y\bigr)
        \mid \mbX = \mbx
    \right].
\]
Therefore,
\[
    \VoI = r_0(\mbx) - r_1(\mbx)
\]
measures the expected reduction in decision loss induced by disclosure
for a decision-maker drawn from the target population $\pi$, rather
than for a particular individual. This assumption does not require
observing the same decision-maker under both regimes. It requires that
the decision-makers observed under the two regimes be representative
of the same target population.

\end{enumerate}

\section{Implementation details}
\label[appendix]{app:impl_details}

\subsection{Synthetic experiments.}
\label[appendix]{app:impl_details_synth}

\paragraph{Dataset generation.}
For our \Synth{} experiments, we consider both a binary classification setting (\SynthB: \(|\mcY| = 2\)) and a multiclass setting (\SynthM: \(|\mcY| = 4\)). In each case, we generate data using \texttt{sklearn.make\_classification} with \texttt{n\_features = 20} (all informative), \texttt{n\_samples = 10000}, and \texttt{class\_sep = 1}. To emulate selective disclosure, we treat a subset of $|\mcX| = 9$ features as \emph{baseline} information and the remaining $|\mcS| = 11$ as \emph{hidden} (i.e., information that is not available in the no-disclosure regime).
Since no real human annotations are available in this setting, we simulate the two decision-makers with logistic regression models: the $D = 0$ decisor is trained on the $9$ baseline features only, while the $D = 1$ decisor is trained on all $20$ features. 
For each instance, the corresponding action $a_i(d)$ is the label predicted by the respective model, so that the $D=0$ decisor is systematically less accurate, and errors concentrate on the instances whose label depends on the hidden features.

\paragraph{Models.}
For the \ClasswiseRisk{} estimator, we first train a multiclass label model \(f\colon \mcX \to \Delta^{|\mcY|}\), implemented as a two-layer MLP with ReLU activations and a final linear layer followed by a softmax over classes, to approximate \(\sP(Y=y \mid X=\mbx)\). In parallel, for each regime \(D\) we train a multi-head error model consisting of a shared backbone (a two-layer MLP with GELU activations) and \(|\mcY|\) binary heads, implemented as a single linear layer of size \(|\mcY|\); the \(y\)-th head takes the shared representation and outputs a logit whose sigmoid corresponds to \(\sP(A(D) \neq y \mid X=\mbx, Y=y)\). Only the head of the observed class contributes to the loss, which is a \(\mathrm{pos\_weight}\)-weighted binary cross-entropy with the same logit correction at prediction time.

For the \TLearner{} estimator, we train two separate scalar risk networks, one per regime \(D \in \{0,1\}\), each modeling \(r_D(\mbx) = \sP(A(D) \neq Y \mid X = \mbx)\). Concretely, each network is a two-layer MLP with GELU activations and dropout, optimized with a weighted binary cross-entropy loss to predict the error indicator \(\mathbbm{1}\{A(D) \neq Y\}\) from the baseline features \(\mcX\). 

The \SLearner{} estimator replaces the two separate risk networks of the \TLearner{} with a single network \(z\colon \mcX \times \{0,1\} \to [0,1]\) that takes the regime as an additional input feature: the binary indicator \(D\) is appended to \(\mbx\), and the network is a two-layer MLP with GELU activations and dropout after each hidden layer, followed by a scalar output. It is trained on the pooled data of both regimes, so that every instance contributes two examples, \((\mbx, 0)\) with target \(\mathbbm{1}\{A(0) \neq Y\}\) and \((\mbx, 1)\) with target \(\mathbbm{1}\{A(1) \neq Y\}\), under a single \(\mathrm{pos\_weight}\)-weighted binary cross-entropy.

Finally, the \Confidence{} baseline does not model the human error at all: it only trains the label model \(g\) of the \ClasswiseRisk{} estimator and ranks instances by its predictive uncertainty, setting \(\hat r_0(\mbx) = u(g(\mbx))\) and \(\hat r_1(\mbx) = 0\), so that \(\widehat{\mathrm{VoI}}(\mbx) = u(g(\mbx))\). We use \(u(\cdot) = 1 - \max_y g(\mbx)_y\) in our experiments. This baseline encodes the intuition that information should be disclosed where the instance is intrinsically hard, regardless of whether the support information helps there; note that, since \(u \geq 0\), its estimated $\hVoI$ is never negative and the policy therefore always spends the full budget, unlike the VoI-based estimators.

For both the \SynthB{} and \SynthM{} settings, all models are trained with Adam for 10 epochs, with weight decay \(10^{-4}\) and without early stopping, using the following hyperparameters:
\begin{itemize}
    \item Label model \(f\) (also used by \Confidence{}): hidden dimension 128, batch size 32, learning rate \(5\times10^{-4}\), no dropout.
    \item Error model \(g_D\): backbone hidden dimension 256, batch size 32, learning rate \(10^{-4}\), no dropout.
    \item \TLearner{} and \SLearner{} risk networks: hidden dimension 128, batch size 128, learning rate \(5\times10^{-4}\), dropout 0.05.
\end{itemize}

\subsection{Real data experiments - \Email{}.}
\label[appendix]{app:impl_details_email}

\paragraph{Email Dataset description.}  
We used a dataset comprising $1000$ real emails with ground-truth labels indicating whether each email was fraudulent (e.g., a phishing attempt) or legitimate. The emails were selected from the corpus introduced by~\citep{DBLP:journals/corr/abs-2511-21448}, which provides only the messages; the human annotations are instead collected in a dedicated behavioural study (see \citealt{bogani2026conceptbottleneckmodelseffective} for details on the email selection procedure). In that study, $300$ participants were recruited online through Prolific and asked to provide binary judgments for $20$ emails randomly sampled from the dataset, indicating whether they considered each email to be fraudulent or legitimate. Participants were randomly assigned to one of two disclosure regimes. In the no-disclosure condition ($D = 0$), participants were presented only with the email to be classified. In the disclosure condition ($D = 1$), they were additionally shown the prediction of a machine-learning model. 

Notably, in this dataset, exchangeability of decision-makers is supported by the random assignment of participants to the two disclosure regimes. We process these data in two steps:

(i) We first aggregate the raw annotations into a per-email summary: each row corresponds to a single email and includes the ground-truth label and the empirical human response distribution in each regime. Concretely, for each label $y \in \{fraudulent, legitimate\}$ we compute \texttt{human\_{0y}} as the fraction of $D=0$ participants who assigned the email to class $y$, and analogously \texttt{human\_{1y}} for the $D=1$ participants. This yields, for each email and regime, an estimated human probability distribution over the two outcomes, from which we derive the human error indicators $err_0$ and $err_1$ (whether a label sampled from the corresponding distribution differs from the ground truth) used in our experiments.

(ii) Then, we augment the train dataset to increase diversity and robustness. We leverage text-based augmentations provided by~\citep{DBLP:journals/corr/abs-2511-21448}~\footnote{We use the augmentations released at \href{https://github.com/DataPhish/PhishFuzzer/blob/main/DataSet_Creation/5_RePhrase/emails_expanded_2026_Gemini.json}{emails-expanded-2026-Gemini.json}, which were generated via prompting a large language model.} and complement them with back-translation. For back-translation, we translate each email into seven languages (fr, de, es, it, pt, nl, ru) and then translate it back to English, obtaining multiple paraphrased variants per original message. In total, this procedure yields up to 14 augmented versions per email and results in $15000$ samples used in our experiments.

\paragraph{Email models.}
For the \Email{} dataset, all models operate on precomputed text embeddings rather than on raw text. Each email is encoded with a publicly available transformer-based phishing-detection model~\footnote{\href{https://huggingface.co/mikaelnurminen/phishing-email-detector-v51}{phishing-email-detector-v51}}, fine-tuned for binary phishing classification: we truncate the email body to $512$ tokens and take the final-layer \texttt{[CLS]} representation, yielding a $768$-dimensional embedding per (original or augmented) email. On top of these embeddings, every network uses the same MLP trunk: a $\mathrm{LayerNorm}$ on the input followed by two $\mathrm{GELU}$ hidden layers with dropout after each layer (halved on the second), where the second hidden layer has half the width of the first. The label model $f$ (also used by \Confidence{}) and the \TLearner{} risk networks use hidden widths $768 \!\to\! 128 \!\to\! 64$, the error model $g_D$ of \ClasswiseRisk{} uses $768 \!\to\! 256 \!\to\! 128$, and the \SLearner{} network, whose input includes the regime indicator, uses $769 \!\to\! 128 \!\to\! 64$. The label and error models are trained with batch size $64$, the \TLearner{} and \SLearner{} networks with batch size $128$ on inputs standardized with training-split statistics.

Model selection is performed by grid search over the number of training epochs $\{25, 50, 100\}$, the optimizer (\textsc{Adam}, \textsc{AdamW}), the learning rate $\{1\times10^{-4}, 2\times10^{-4}\}$, the dropout rate $\{0.1, 0.3\}$ and the weight decay $\{10^{-3}, 10^{-2}\}$, retaining for each model the configuration with the highest binary AUROC on the validation split. The resulting hyperparameters are reported in Table~\ref{tab:email_hparams}.

\begin{table}
\centering
\caption{Selected hyperparameters for the \Email{} risk models.}
\label{tab:email_hparams}
\small
\begin{tabular}{lccccc}
\toprule
Model & Epochs & Optimizer & LR & Dropout & Weight decay \\
\midrule
\ClasswiseRisk{} & $50$  & \textsc{Adam}  & $2\times10^{-4}$ & $0.3$ & $10^{-3}$ \\
\TLearner{}      & $25$  & \textsc{AdamW} & $2\times10^{-4}$ & $0.1$ & $10^{-3}$ \\
\SLearner{}      & $25$  & \textsc{Adam}  & $1\times10^{-4}$ & $0.1$ & $10^{-2}$ \\
\Confidence{}    & $100$ & \textsc{AdamW} & $2\times10^{-4}$ & $0.3$ & $10^{-3}$ \\
\bottomrule
\end{tabular}
\end{table}

\subsection{Real data experiments - \Image{}.}
\label[appendix]{app:impl_details_cp}

\paragraph{\Image{} Dataset description.}
We used the \textsc{\Image{}} dataset introduced by~\citep{steyvers2022bayesian}, as processed for the human study of~\citep{DBLP:conf/nips/ToniOTSR24}. The dataset comprises 1200 natural images drawn from 16 categories $\{$\texttt{airplane}, \texttt{bear}, \texttt{bicycle}, \texttt{bird}, \texttt{boat}, \texttt{bottle}, \texttt{car}, \texttt{cat}, \texttt{chair}, \texttt{clock}, \texttt{dog}, \texttt{elephant}, \texttt{keyboard}, \texttt{knife}, \texttt{oven}, \texttt{truck}$\}$, each corrupted with phase noise to make the classification task challenging for both humans and models; we use the highest available noise level ($110$), which yields the hardest regime. Each image carries a ground-truth category label. A VGG-19 classifier fine-tuned on the noisy images provides softmax scores over the 16 classes, from which~\citep{DBLP:conf/nips/ToniOTSR24} construct conformal prediction sets using the standard split-conformal calibration procedure.

Concretely, given a target miscoverage level $\alpha$, the conformal score of an image is defined as $s = 1 - \hat{p}(y \mid x)$, where $\hat{p}(y \mid x)$ is the model's softmax probability of the ground-truth class $y$~\citep{DBLP:journals/sigact/Law06}. The threshold $q$ is set to the $(1-\alpha)(n+1)/n$ empirical quantile of these scores on a calibration set, and the prediction set for an image is the collection of classes whose softmax probability exceeds $q$. This construction enjoys the usual marginal coverage guarantee of $1-\alpha$: smaller values of $\alpha$ enforce higher coverage and therefore yield larger prediction sets, whereas larger values of $\alpha$ produce smaller, more informative sets~\citep{DBLP:journals/corr/SadinleLW16}. Since the human study recorded, for each image, the responses of participants who were shown a prediction set of a given size, the choice of $\alpha$ determines both the set displayed in the disclosure condition and the corresponding empirical human response distribution. In our experiments we consider three miscoverage levels, $\alpha \in \{0.01, 0.02, 0.05\}$. We will discuss in \cref{app:alpha} about the choice of $\alpha$. 

Human judgments were collected online in two disclosure regimes, drawn from two separate behavioural studies over the same $1200$ images. In the no-disclosure condition ($D = 0$), taken from~\citep{steyvers2022bayesian}, participants were shown only the (noisy) image and asked to select one of the $16$ categories, yielding $7261$ classifications from $145$ participants. In the disclosure condition ($D = 1$), taken from~\citep{DBLP:conf/icml/StraitouriR24}~\footnote{Datasets for both conditions are publicly available: \href{https://osf.io/2ntrf/overview?view_only=9ec9cacb806d4a1ea4e2f8acaada8f6c}{no-disclosure condition} and \href{https://github.com/Human-Centric-Machine-Learning/towards-human-ai-complementarity-predictions-sets}{disclosure condition}.}, participants were additionally shown the model's conformal prediction set for the image before selecting a label. The set was presented in a \emph{lenient} regime: it served as a suggestion rather than a constraint, so participants were free to choose any of the $16$ categories, including labels outside the displayed set ~\footnote{The original study also includes a \emph{strict} condition, in which participants are constrained to choose a label from within the displayed prediction set. We deliberately adopt the lenient regime, as our setting requires the human to remain the final decision-maker and to be free to disregard the model's feedback whenever their own judgment disagrees with it.}

Notably, exchangeability of decision-makers for this dataset is plausible because: $(i)$ the two studies use the same task and item population and $(ii)$ recruit participants from comparable populations under similar experimental protocols, with the principal difference being the availability of support information.
We process the human study data in two steps:

(i) We first aggregate the raw annotations into a per-image summary: each row corresponds to a single image and includes the ground-truth label, the model softmax scores, and the human responses collected under each regime. For the no-disclosure regime, for each category $y$ we compute \texttt{human\_{0y}} as the fraction of $D=0$ participants who assigned the image to class $y$, and a human label is then obtained by sampling from this distribution.
For the disclosure regime, the response we use depends on the chosen miscoverage level $\alpha$. For each image, the experiment of \citep{DBLP:conf/icml/StraitouriR24} collected disclosure responses from 16 distinct $D=1$ participants, each of whom was shown a prediction set of a different size, ranging from 1 to 16 labels. Since the size of the conformal prediction set for a given image is determined by $\alpha$, each of these responses corresponds to the decision a human would make when assisted at a particular miscoverage level. Fixing $\alpha$ therefore determines, for each image, the size of the set the policy would actually disclose, and we retain only the response of the participant who was shown a set of exactly that size.
This yields, for each image and regime, a human label (sampled from the empirical distribution in the no-disclosure regime, or the human response selected via $\alpha$ in the disclosure regime) from which we derive the human error indicators $err_0$ and $err_1$ used in our experiments.

(ii) Then, we augment the train dataset to increase diversity and robustness. At training time each image is transformed by composing 3 randomly sampled geometric augmentations drawn from a fixed set of $14$ operations (identity, horizontal/vertical flips, rotations of $15^\circ$/$30^\circ$/$45^\circ$, horizontal/vertical translations, and $x$/$y$ shears), yielding multiple perturbed variants per original image while preserving its category label and associated human response distributions.
Figure~\ref{fig:cat_ex} illustrates the resulting variants for a representative image.

\begin{figure}[ht]
    \centering
    \includegraphics[width=\linewidth]{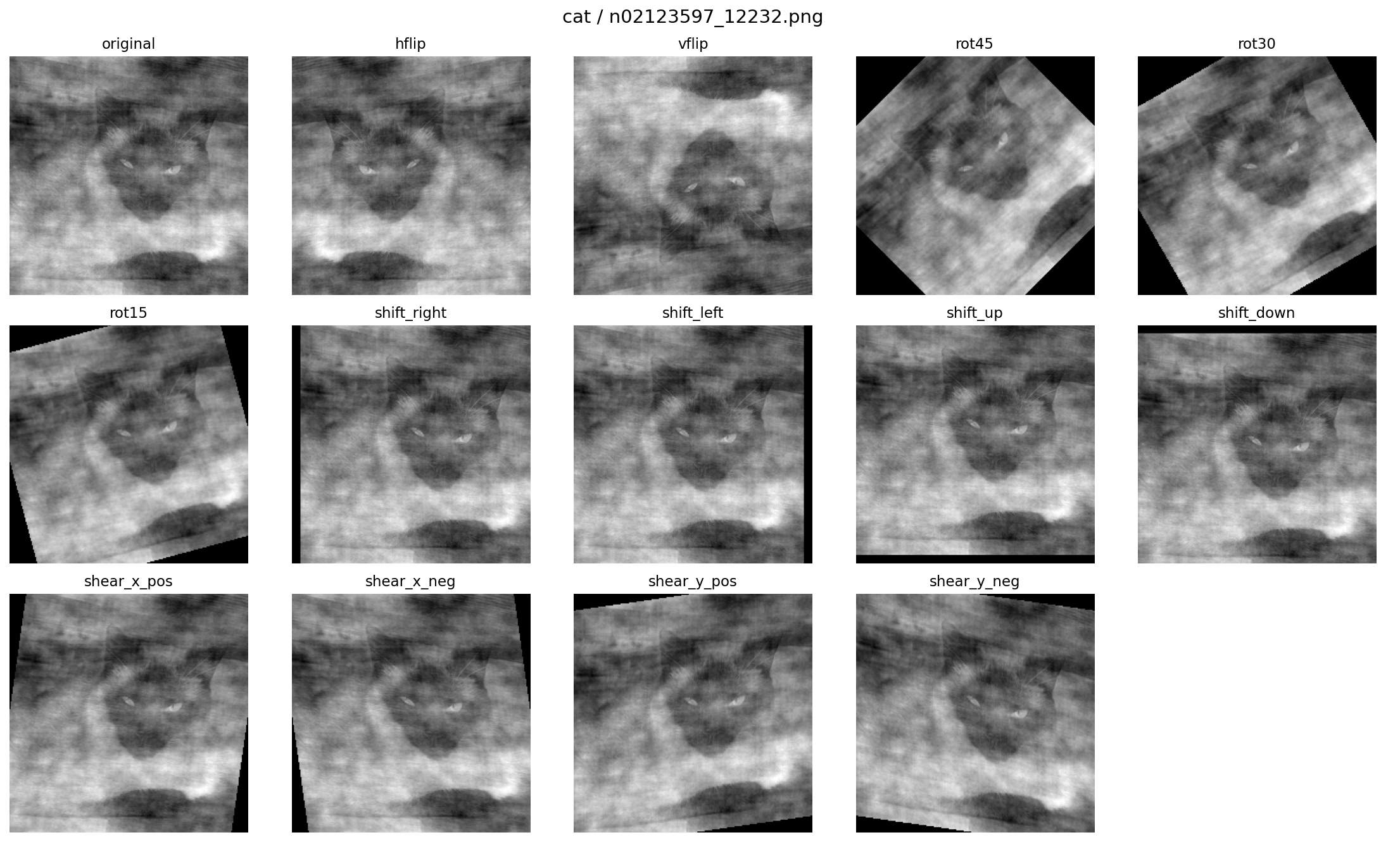}
    \caption{Geometric augmentations applied to the \Image{} images at training time,    shown here for one image of the \texttt{cat} category at the highest phase-noise level ($110$).    Each panel corresponds to one of the $14$ operations in the augmentation pool: the identity    (\emph{original}), horizontal and vertical flips, rotations of $15^\circ$, $30^\circ$ and    $45^\circ$, translations along the four directions, and positive/negative shears along the    $x$ and $y$ axes. }    
    \label{fig:cat_ex}
\end{figure}

\paragraph{Exchangeability across studies.}
Unlike \Email{}, where participants were randomly assigned to regimes within a single study, the two \Image{} regimes come from different studies. Both use the same $1200$ images at noise level $110$, but the protocols slightly differ beyond the availability of support: $D=0$ participants classified $200$ images each, with a confidence rating on every trial, from a pool spanning four noise levels~\citep{steyvers2022bayesian}, whereas $D=1$ participants were recruited on Prolific and answered multiple-choice questionnaires at noise level $110$ only~\citep{DBLP:conf/icml/StraitouriR24}. 

We assess the plausibility of the exchangeability assumption using a feature of the design by ~\citep{DBLP:conf/icml/StraitouriR24}. Their study includes trials in which the displayed prediction set contains the entire label space $\mathcal{Y}$. Because such a set provides no information about which label is more likely, performance on these trials provides a useful proxy for no-disclosure performance. Restricting the comparison to the same images, accuracy on full-set trials is $0.7713$, compared with $0.7647$ under $D=0$. The similarity of these accuracies provides descriptive evidence that the decision-maker populations in the two studies have comparable baseline performance.

We acknowledge that this comparison does not establish exchangeability, which cannot be determined from the observed data alone. Moreover, displaying a nondiscriminative prediction set is not identical to withholding support. Nevertheless, together with the shared task and image set, this evidence supports the plausibility of treating participants in the two studies as draws from a common target population $\pi$.

\paragraph{\Image{} Models.}
For the \Image{} dataset, all models operate directly on the (noisy) images. Every network uses the same feature extractor, an ImageNet-pretrained \textsc{ConvNeXt-Tiny} backbone~\citep{DBLP:journals/corr/abs-2201-03545} that maps each image to a $768$-dimensional embedding and is kept frozen during training, followed by an MLP head with an input $\mathrm{LayerNorm}$, two $\mathrm{GELU}$ hidden layers and dropout after each (halved on the second). The hidden widths are $128 \!\to\! 256$ for the label model $f$ (also used by \Confidence{}, and trained with label smoothing $0.05$), $256 \!\to\! 256$ for the error model $g_D$ of \ClasswiseRisk{}, and $128 \!\to\! 256$ for the \TLearner{} and \SLearner{} risk networks, the latter taking the regime indicator as an additional input. Losses and logit corrections are as in the synthetic setting. The label and error models are trained with batch size $64$, the \TLearner{} and \SLearner{} networks with batch size $128$.

For all four models, model selection is performed by grid search over the same hyperparameter space described above for \Email{} (epochs, optimizer, learning rate, dropout rate, and weight decay), tuning each model separately for every miscoverage level $\alpha \in \{0.01, 0.02, 0.05\}$ and retaining the configuration with the highest binary AUROC on the validation split. The resulting hyperparameters are reported in Table~\ref{tab:cp_hparams}.

\begin{table}
\centering
\caption{Selected hyperparameters for the \Image{} risk models, per conformal miscoverage level $\alpha$.}
\label{tab:cp_hparams}
\small
\begin{tabular}{llccccc}
\toprule
$\alpha$ & Model & Epochs & Optimizer & LR & Dropout & Weight decay \\
\midrule
\multirow{4}{*}{$0.01$}
 & \TLearner{}      & $25$  & \textsc{AdamW} & $2\times10^{-4}$ & $0.1$ & $10^{-3}$ \\
 & \ClasswiseRisk{} & $25$  & \textsc{AdamW} & $1\times10^{-4}$ & $0.3$ & $10^{-2}$ \\
 & \SLearner{}      & $25$  & \textsc{AdamW} & $2\times10^{-4}$ & $0.1$ & $10^{-3}$ \\
 & \Confidence{}    & $25$  & \textsc{Adam}  & $1\times10^{-4}$ & $0.3$ & $10^{-2}$ \\
\midrule
\multirow{4}{*}{$0.02$}
 & \TLearner{}      & $25$  & \textsc{AdamW} & $2\times10^{-4}$ & $0.1$ & $10^{-3}$ \\
 & \ClasswiseRisk{} & $25$  & \textsc{AdamW} & $1\times10^{-4}$ & $0.1$ & $10^{-3}$ \\
 & \SLearner{}      & $25$  & \textsc{Adam}  & $1\times10^{-4}$ & $0.3$ & $10^{-3}$ \\
 & \Confidence{}    & $25$  & \textsc{Adam}  & $1\times10^{-4}$ & $0.1$ & $10^{-3}$ \\
\midrule
\multirow{4}{*}{$0.05$}
 & \TLearner{}      & $25$  & \textsc{AdamW} & $1\times10^{-4}$ & $0.3$ & $10^{-3}$ \\
 & \ClasswiseRisk{} & $100$ & \textsc{AdamW} & $1\times10^{-4}$ & $0.3$ & $10^{-2}$ \\
 & \SLearner{}      & $100$ & \textsc{Adam}  & $1\times10^{-4}$ & $0.3$ & $10^{-2}$ \\
 & \Confidence{}    & $25$  & \textsc{Adam}  & $1\times10^{-4}$ & $0.1$ & $10^{-3}$ \\
\bottomrule
\end{tabular}
\end{table}

\section{User studies}
\label[appendix]{app:hum_exp}

\paragraph{Procedure.}

For both experiments, we recruit participants through Prolific and randomly assign them to one of the four experimental conditions (Human Policy/Low Budget, \emph{HP-LB}; Human Policy/High Budget, \emph{HP-HB}; Machine Policy/Low Budget, \emph{MP-LB}; Machine Policy/High Budget, \emph{MP-HB}).~\footnote{We did not include a random-disclosure condition because it does not address the primary question of these user studies: comparing the overall effectiveness of selective and human-determined disclosure, independently of the specific factors that may drive the results. A random condition would instead provide a diagnostic comparison of whether machine-selected items lead to better performance compared to an equal number of randomly selected items. We evaluate this selection question directly in Q1, where Random is matched to the learned policy at each budget. Accordingly, the user-study results compare the overall performance in MP and HP condition (without aiming to attribute any difference uniquely to specific factors such as item selection, disclosure frequency, or participant agency). Furthermore, crossing a Random condition with both budget levels would expand the design from four to six conditions, requiring substantially more participants to preserve per-condition statistical power.} First, we provide them with task instructions and have them complete two practice trials to familiarize themselves with the interface (see examples of the interfaces used in both experiments in \cref{fig:interfaces}). Both the instructions and the practice trials are tailored to the experimental condition to which participants are assigned. To incentivize attentive responding, we inform participants that the 5 most accurate participants would receive a bonus of £5.00. We also tell participants in the HP conditions that, in the case of a tie, participants who used fewer requests for assistance would be favored to receive the bonus. This is done to prompt participants to request assistance only when they deem it necessary, similarly to what occurs when the provision of assistance is determined by the \ClasswiseRisk{} policy in the MP conditions, which not necessarily exhaust the available budget.

Participants then proceed to the actual experiment, which consists of a total of 20 trials. In each trial, participants have to classify an item (an email in the \Email{} task or an image in the \Image{} one).~\footnote{Since we planned to recruit participants from the broader population on Prolific, it was important that the tasks did not require specialized knowledge. In this respect, both \Email{} and \Image{} tasks involve classification tasks that can be performed without domain-specific expertise.} If they are assigned to one of the HP conditions, they can request assistance from the AI system (provided that they have not exhausted the available requests); if they are assigned to one of the MP conditions, depending on the \ClasswiseRisk{} policy output, they are either provided with assistance from the AI system or informed that no assistance would be provided for that item.~\footnote{We note that, for both task, the AI advice provided is the same as the one presented in the annotation studies from which the test sets for the present work were derived.} After classifying the item, we also ask participants to indicate how confident they are that their answer is correct on a 7-point Likert scale ranging from 1 (“Not confident at all (guessing)”) to 7 (“Extremely confident”). No feedback is provided to participants on the accuracy of their responses. This both reflects realistic classification settings, in which the correctness of a classification may not be immediately verifiable, and limits learning across trials. Throughout the 20 trials, we include two attention checks and record the number of times participants switch away from the experiment browser tab. 

The 20 items presented to each participant are randomly sampled from each dataset's test set using stratified sampling. In both the HP and MP conditions, sampling preserves, as closely as possible, the proportions observed in the full test set along two characteristics: whether the selective disclosure method would provide support information for an item and whether such information is useful (i.e., whether, in the human-annotation studies used to construct the original datasets, participants classified that item more accurately with support information than without it). This procedure ensures that the samples presented to participants are as representative as possible of the original test set with respect to these characteristics. In the \Email{} task, we further stratified the sampling according to the ground truth of the emails to be classified, in order to avoid item samples in which the emails were predominantly fraudulent or legitimate. Indeed, presenting too many items of the same class to participants during classification tasks may cause them to respond incorrectly simply because a prolonged sequence of identical answers is perceived as unnatural, leading them to change an otherwise correct answer~\citep{pesenti2026samealgorithmichumanbias}. This risk is particularly relevant in binary classification tasks such as \Email{}, but substantially less so when multiple classes are present, as in \Image{}; we therefore do not apply this additional stratification to the latter. Finally, we specify that, within each experiment, we sample items from the same pool across all four experimental conditions; what differs between the LB and HB conditions is whether the method provides assistance for a given item, depending on the budget level. See \cref{tab:sampling_stratification} for details on the item-sampling stratification.~\footnote{We excluded some items from the \Email{} test set because they were near-duplicates of other emails, and from the \Image{} test set because they were potentially ambiguous (e.g., images with \emph{bottle} as the ground-truth class that actually depicted vases). This left 184 emails (original test-set size: 200) and 220 images (original test-set size: 241), from which the items presented to participants were sampled.}

\begin{table}[ht]
\centering
\caption{Stratification of the 20-item samples used in Experiments 1 and 2.}
\label{tab:sampling_stratification}
\small

\resizebox{\textwidth}{!}{%
\begin{tabular}{lllcc}
\toprule
& & & \multicolumn{2}{c}{\textbf{Budget}} \\
\cmidrule(lr){4-5}

\textbf{support information} &
\textbf{Usefulness of assistance} &
\textbf{Ground truth} &
\textbf{Low (30\%)} &
\textbf{High (70\%)} \\
\midrule

\multicolumn{5}{l}{\Image{}} \\[3pt]

\multicolumn{5}{l}{\textit{20-item sample characteristics}} \\[2pt]

Not provided & Not useful & --- & 12 (62.2\%) & 6 (29.0\%) \\
Provided     & Not useful & --- & 4 (24.9\%)  & 12 (58.1\%) \\
Not provided & Useful     & --- & 2 (4.6\%)   & 0 (0.4\%) \\
Provided     & Useful     & --- & 2 (8.3\%)   & 2 (12.4\%) \\[3pt]

\multicolumn{3}{l}{Items with support information in 20-item sample}
    & 6/20 (30\%) & 14/20 (70\%) \\[4pt]

\multicolumn{5}{l}{\textit{Original test-set characteristics}} \\[2pt]

\multicolumn{3}{l}{Test-set size}
    & \multicolumn{2}{c}{241} \\

\multicolumn{3}{l}{Policy provides assistance}
    & 80/241 (33\%) & 170/241 (71\%) \\

\midrule

\multicolumn{5}{l}{\Email{}} \\[3pt]

\multicolumn{5}{l}{\textit{20-item sample characteristics}} \\[2pt]

Not provided & Not useful & Legitimate & 6 (30.5\%) & 4 (20.0\%) \\
Provided     & Not useful & Legitimate & 2 (8.0\%)  & 4 (18.5\%) \\
Not provided & Useful     & Legitimate & 2 (13.0\%) & 1 (9.0\%) \\
Provided     & Useful     & Legitimate & 1 (5.5\%)  & 2 (9.5\%) \\
Not provided & Not useful & Fraudulent & 5 (23.0\%) & 4 (18.5\%) \\
Provided     & Not useful & Fraudulent & 1 (6.5\%)  & 2 (11.0\%) \\
Not provided & Useful     & Fraudulent & 2 (9.5\%)  & 2 (6.5\%) \\
Provided     & Useful     & Fraudulent & 1 (4.0\%)  & 1 (7.0\%) \\[3pt]

\multicolumn{3}{l}{Items with support information in 20-item sample}
    & 5/20 (25\%) & 9/20 (45\%) \\

\multicolumn{3}{l}{Legitimate emails in 20-item sample}
    & 11/20 (55\%) & 11/20 (55\%) \\[4pt]

\multicolumn{5}{l}{\textit{Original test-set characteristics}} \\[2pt]

\multicolumn{3}{l}{Test-set size}
    & \multicolumn{2}{c}{200} \\

\multicolumn{3}{l}{Policy provides assistance}
    & 48/200 (24\%) & 92/200 (46\%) \\

\multicolumn{3}{l}{Legitimate emails}
    & 114/200 (57\%) & 114/200 (57\%) \\

\bottomrule
\end{tabular}%
}

\vspace{6pt}

\begin{minipage}{0.96\linewidth}
~\footnotesize
\textit{Note.} For the 20-item sample characteristics, values outside parentheses
indicate the number of items included in the sample, while percentages in
parentheses indicate the corresponding proportion of items in the original
test set.
\end{minipage}
\end{table}

After having completed all 20 trials, we ask participants to complete a trust scale from~\citep{DBLP:journals/fcomp/HoffmanMKL23}, consisting of eight 5-point Likert scale items (extremes: I strongly disagree; I strongly agree). The items, presented in randomized order, were the following:

\begin{itemize}
    \item I am confident in the system. I feel that it works well.
    \item The outputs of the system are very predictable.
    \item The system is very reliable. I can count on it to be correct all the time.
    \item I feel safe that when I rely on the system I will get the right answers.
    \item The system is efficient in that it works very quickly.
    \item I am wary of the system.
    \item The system can perform the task better than a novice human user.
    \item I like using the system for decision making.
\end{itemize}

Finally, we ask participants to report their familiarity with AI systems by selecting one of the following options:

\begin{itemize}
    \item \textbf{Option 1}: I have little or no experience with AI systems and limited or no understanding of how they work.
    \item \textbf{Option 2}: I use AI systems occasionally but have not a clear understanding of how they function.
    \item \textbf{Option 3}: I use AI systems and have studied how they work (e.g., through courses, online classes, or self‑study).
    \item \textbf{Option 4}: I develop or build AI systems as part of my work or personal projects.
\end{itemize}

\begin{figure}[ht]
    \centering
    \includegraphics[width=0.855\linewidth]{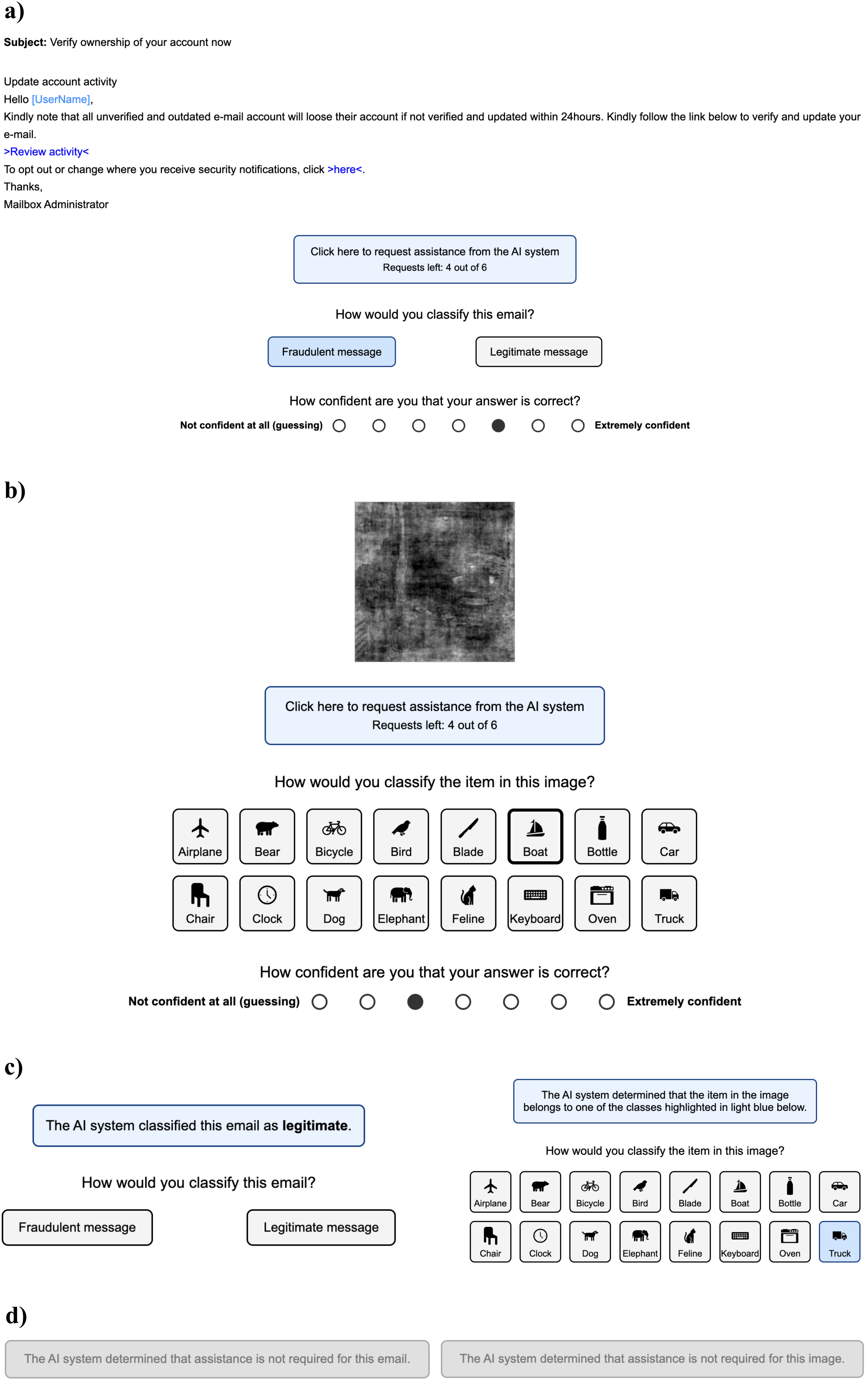}
    \caption{\textit{(a)} and \textit{(b)}: Interfaces presented to participants in the HP conditions in the \Email{} and \Image{} tasks, respectively; \textit{(c)}: Messages shown when information was requested (HP) or provided by the \ClasswiseRisk{} policy (MP); \textit{(d)}: Messages shown when, in MP, the \ClasswiseRisk{} policy determined information should not be provided for that item.}
    \label{fig:interfaces}
\end{figure}

All datasets generated from the human-subject experiments and analysis scripts can be found in the \href{https://anonymous.4open.science/r/submissionICLR2027-4163}{anomymous repo.} 

\paragraph{Participants recruitment and samples characteristics}
We conducted an a priori power analysis using a simulation-based approach~\citep{green2016simr, kumle2021estimating} to estimate the sample size to collect. The analysis indicated that a total sample of 268 participants would have provided 82\% statistical power to detect an effect of the two-way interaction between \emph{support} and \emph{budget level} conditions as small as $OR = 0.43$.~\footnote{As detailed later, the majority of the statistical analyses conducted were logistic regressions. Thus, odds ratios correspond to the exponentiated regression coefficients, $exp(b)$.} Accordingly, participants were recruited in batches through Prolific. After each batch, we assessed only whether participants met predetermined exclusion criteria: having passed both attention checks and not having left the browser tab in three or more trials. No further analyses of the outcome variables or effect sizes were performed during data collection. Recruitment continued until the final sample comprised a total of at least 268 participants. Inclusion criteria required participants to be native English speakers from the UK and to have a Prolific approval rate above 98\%. Participants received compensation equal to £1.30. 

For the \Email{} and \Image{} studies we recruited, respectively, a total of 288 and 278 participants, of which 272 and 273 were included in the final sample. In both studies, participants samples were balanced in terms of age, sex, and previous experience with AI across all experimental conditions (see \cref{tab:participant_characteristics}).  

\begin{table}[ht]
\centering
\caption{Participant characteristics by experimental condition.}
\label{tab:participant_characteristics}
\small

\resizebox{\textwidth}{!}{%
\begin{tabular}{lccc|cccc}
\toprule
& & & &
\multicolumn{4}{c}{\textbf{Past experience with AI}} \\
\cmidrule(lr){5-8}

\textbf{Experimental condition} &
\textbf{\textit{n} participants} &
\textbf{Age} &
\textbf{\% Female} &
\textbf{Option 1} &
\textbf{Option 2} &
\textbf{Option 3} &
\textbf{Option 4} \\
\midrule

\multicolumn{8}{l}{\textit{\Image{}}} \\[2pt]

Human Policy/Low Budget
    & 66 & $43.21 \pm 14.05$ & 52\% & 5\% & 67\% & 29\% & 0\% \\

Human Policy/High Budget
    & 69 & $42.06 \pm 12.39$ & 48\% & 3\% & 46\% & 49\% & 1\% \\

Machine Policy/Low Budget
    & 69 & $42.30 \pm 12.23$ & 48\% & 6\% & 54\% & 41\% & 0\% \\

Machine Policy/High Budget
    & 69 & $41.71 \pm 13.52$ & 42\% & 1\% & 65\% & 28\% & 6\% \\

\midrule

\multicolumn{8}{l}{\textit{\Email{}}} \\[2pt]

Human Policy/Low Budget
    & 72 & $44.58 \pm 15.63$ & 51\% & 4\% & 53\% & 37\% & 6\% \\

Human Policy/High Budget
    & 64 & $39.33 \pm 11.72$ & 48\% & 0\% & 64\% & 34\% & 2\% \\

Machine Policy/Low Budget
    & 71 & $41.87 \pm 13.15$ & 55\% & 4\% & 62\% & 31\% & 3\% \\

Machine Policy/High Budget
    & 65 & $41.98 \pm 13.07$ & 55\% & 3\% & 49\% & 48\% & 0\% \\

\bottomrule
\end{tabular}%
}

\end{table}

\section{Extended results}
\label[appendix]{app-subsec:extendedHres}
In this section, we detail the results obtained for both the analyses presented in the main text and additional ones we present here. We organize the analyses around the three research questions of the main text, together with a fourth one that we address here:

\begin{itemize}

  \item[\textbf{Q1}]
  Does our approach learn effective policies?
  \item[\textbf{Q2}] 
  Does our learned policy improve human-AI team performance?
  \item[\textbf{Q3}] 
  How do participants interact with our policy?
  \item[\textbf{Q4}]
  How does changing the information content affect our policy?
\end{itemize}

For questions related to the user studies (Q2 and Q3), we report the full statistical results.
Many of the analyses we performed for such questions involved fitting logistic mixed-effects regression models. These models extend standard regression by including random effects that account for the non-independence of repeated observations (e.g., multiple responses from the same participant or to the same item), thereby yielding valid inferences despite correlated observations (for an overview, see \citealp{brown2021introduction}). For all mixed-effects models reported, we include random intercepts for participants and items. 

In these regressions, categorical predictors are deviation-coded as $+0.5/-0.5$ (the level coded as $+0.5$ is reported in brackets in the relevant results tables). This coding allows coefficients to be interpreted as comparisons averaged across the levels of the other categorical predictors. When we perform pairwise post-hoc comparisons, we apply Bonferroni corrections to the \textit{p} values to control the family-wise error rate at a significance level of .05 (we report corrected \textit{p} values throughout). 

\subsection{Q1: Does our approach learn effective policies?} 
\label[appendix]{app:q1}

\textbf{Take-home message}: \ClasswiseRisk{} ranks instances better and abstains more selectively than standard meta-learners, and its ability to abstain matters most when disclosure is harmful on a substantial fraction of the instances.

In the main paper, \ClasswiseRisk{} outperforms the baselines across all datasets (\cref{fig:results}). Here we investigate \emph{why}, through two complementary analyses. First, we isolate the contribution of the class-wise decomposition by replacing it with two standard meta-learners while keeping the disclosure rule fixed: any difference in accuracy can then only come from the quality of the VoI estimates. Second, we look at \emph{how} the policies spend their budget: since the policy never discloses on instances with non-positive estimated VoI, the fraction of instances on which it actually requests help reveals whether an estimator can identify where disclosure does not help.

\paragraph{Ablation on the risk estimator.}

We compare our \ClasswiseRisk{} estimator against two standard meta-learners for conditional treatment effects, the \TLearner{} and the \SLearner{}~\citep{kunzel2019metalearners}. We do not include doubly robust learners~\citep{10.1214/23-EJS2157} or causal forests~\citep{Wager03072018}. DR learners rely on cross-fitted nuisance models, which would further split our small real datasets, while their correction matters little here since disclosure is assigned independently of the covariates. Causal forests, on the other hand, struggle with the high-dimensional representations of the unstructured inputs (text and images) of our real datasets.
\ClasswiseRisk{} estimates $\VoI$ by first predicting the label and then the class-conditional probability of a human error (\cref{prop:class-wise}). Here we ask how much of its performance is due to this decomposition, as opposed to the disclosure rule itself. We therefore keep the policy, the calibration procedure and the budget grid fixed, and vary only how the two regime-specific risks are estimated. All estimators are tuned over the same hyperparameter grid.

The \TLearner{} fits one scalar risk network per regime directly on the error indicator, $\hat r_d(\mbx) \approx \sP(A(d) \neq Y \mid X = \mbx)$, trained on the targets $\mathbbm{1}\{A(d) \neq Y\}$, and sets $\hVoI = \hat r_0(\mbx) - \hat r_1(\mbx)$. \ClasswiseRisk{} is itself a T-learner; the \TLearner{} considered here differs only in modelling each risk with a single scalar output rather than through the class-wise decomposition. Since the two regimes get independent functions, the risk curves are free to take different shapes and $\hVoI$ is not systematically biased towards zero, at the cost of a higher variance, as it is the difference of two independently estimated quantities.

The \SLearner{} fits a single network that takes the regime as an additional input feature, $\hat r(\mbx, d) \approx \sP(A(d) \neq Y \mid X = \mbx)$, trained on the pooled data of both regimes, and sets $\hVoI = \hat r(\mbx, 0) - \hat r(\mbx, 1)$. All parameters are shared across regimes, which reduces variance but tends to shrink the effect of $d$, and hence $\hVoI$, towards zero.
\begin{figure}[t]
    \centering
    \includegraphics[width=\linewidth]{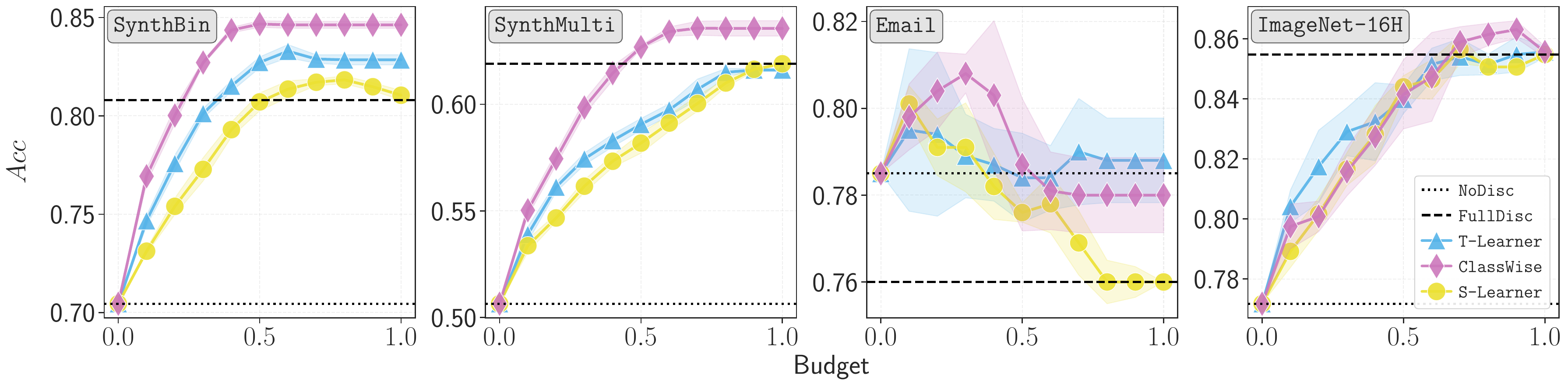}
    \caption{Budget-accuracy curves for VoI-based disclosure policies on \Synth{} (\SynthB{}, \SynthM{}) and real (\Email{}, \Image{}) data. We report average mean values for $5$ seeds and $95\%$ confidence intervals. }
    \label{appfig:results_ablations}
\end{figure}
\cref{appfig:results_ablations} reports the budget--accuracy curves. 

On \Synth{} data the ordering is stable across budgets: \ClasswiseRisk{} $>$ \TLearner{} $>$ \SLearner{} at every $B$, in both the \SynthB{} and the \SynthM{} setting. The gap is substantial and does not close with the budget: in the \SynthB{} case \ClasswiseRisk{} plateaus at $\approx.85$ against $\approx.83$ for \TLearner{} and $\approx.81$ for \SLearner{}; in the \SynthM{} case both meta-learners converge to full-disclosure accuracy ($\approx.62$) without ever exceeding it, whereas \ClasswiseRisk{} exceeds it from $B\approx.4$. 
Since the disclosure rule is identical, the difference is entirely attributable to the quality of the $\VoI$ estimate, both its ranking and its sign. The \SynthM{} panel is where the class-wise decomposition pays off most: disclosure affects the $|\mcY|$ classes heterogeneously, and while a scalar risk model can in principle represent this heterogeneity, the decomposition makes it explicit, so that each head only has to learn the error pattern of a single class.

On \Email{}, where disclosure is harmful on average, \ClasswiseRisk{} attains the highest accuracy overall ($\approx.81$ at $B=.30$) but \TLearner{} is better at large budgets ($\approx.79$ against $\approx.78$ for $B\geq.70$). \SLearner{}, instead, degrades monotonically after $B=.10$ and coincides with full disclosure from $B=.80$ onwards. This is the shrinkage of the \SLearner{} made visible: once $\hVoI$ is compressed towards zero it loses the information about its sign and ends up almost uniformly positive, so the non-negativity constraint of \cref{thm:optimal_disclosure} stops binding and the policy loses the ability to abstain precisely in the regime where abstention is the whole point.

On \Image{}, where disclosure helps almost everywhere, the picture reverses: abstention is nearly irrelevant, all three estimators improve steadily, and \SLearner{} is no longer penalised, tracking \ClasswiseRisk{} closely up to $B=.60$. Here \TLearner{} leads at low budgets (e.g.\ $\approx.82$ against $\approx.81$ at $B=.20$) while \ClasswiseRisk{} takes over from $B\approx.60$ and peaks at $\approx.86$ at $B=.90$, above every other curve and above full disclosure ($\approx.85$).

\paragraph{Help-request frequency and quality of the risk estimates.}
\label[appendix]{app:help_freq}

A useful diagnostic of a disclosure policy is not only \emph{how much} accuracy it attains, but \emph{how much of the budget it actually spends} to attain it. By \cref{thm:optimal_disclosure}, the policy never discloses on instances with non-positive estimated VoI. Once the budget exceeds the fraction of instances with $\hVoI > 0$, additional budget triggers no further disclosures, and the fraction of instances on which the policy requests help plateaus below the nominal budget. The height of this plateau measures the size of the positive-VoI set identified by each estimator, i.e., the fraction of instances it deems worth disclosing. \cref{fig:cp_freq_budget} reports this fraction as a function of the budget for all datasets.

The \SLearner{} stays on the diagonal $\text{spent} = \text{budget}$ in every setting: $\hVoI \ge 0$ on essentially every instance, so the sign constraint of \cref{thm:optimal_disclosure} never binds and the policy never abstains. By sharing all parameters across regimes, the \SLearner{} compresses $\hVoI$ towards a small, almost uniformly positive value and loses the information about its sign. The comparison of interest is therefore between \TLearner{} and \ClasswiseRisk{}, which both abstain, but to different extents depending on the dataset.

On \SynthB{} and \SynthM{}, \ClasswiseRisk{} abstains earlier and more often than \TLearner{}. \ClasswiseRisk{} plateau is at about $54\%$ of the instances in \SynthB{} and $71\%$ in \SynthM{}, against $69\%$ and $83\%$ for \TLearner{}. Since \ClasswiseRisk{} also attains the highest accuracy at every budget (\cref{fig:results}, \cref{appfig:results_ablations}), its smaller positive-VoI set does not come from missing useful disclosures, but from correctly withholding the support information where it would not help.

On \Email{}, disclosure is harmful on average, and both estimators abstain on a large fraction of instances: \TLearner{} plateau is $\approx 66\%$, \ClasswiseRisk{} is $\approx 53\%$. The more parsimonious policy of \ClasswiseRisk{} attains the highest accuracy overall, at intermediate budgets (\cref{fig:results}, \cref{appfig:results_ablations}). At large budgets, the additional disclosures of \TLearner{} yield a slightly higher accuracy, suggesting that \ClasswiseRisk{} withholds a small set of instances on which disclosure would still have helped.

On \Image{} disclosure is beneficial for almost every instance, and both estimators disclose almost up to the full budget: at $B = 1$, \ClasswiseRisk{} discloses to $98\%$ of the instances and \TLearner{} to nearly all of them. With few instances to withhold, the two estimators differ only in the \emph{ordering} of the instances, not in the size of the positive-VoI set.

\begin{figure}
    \centering
    \includegraphics[width=\linewidth]{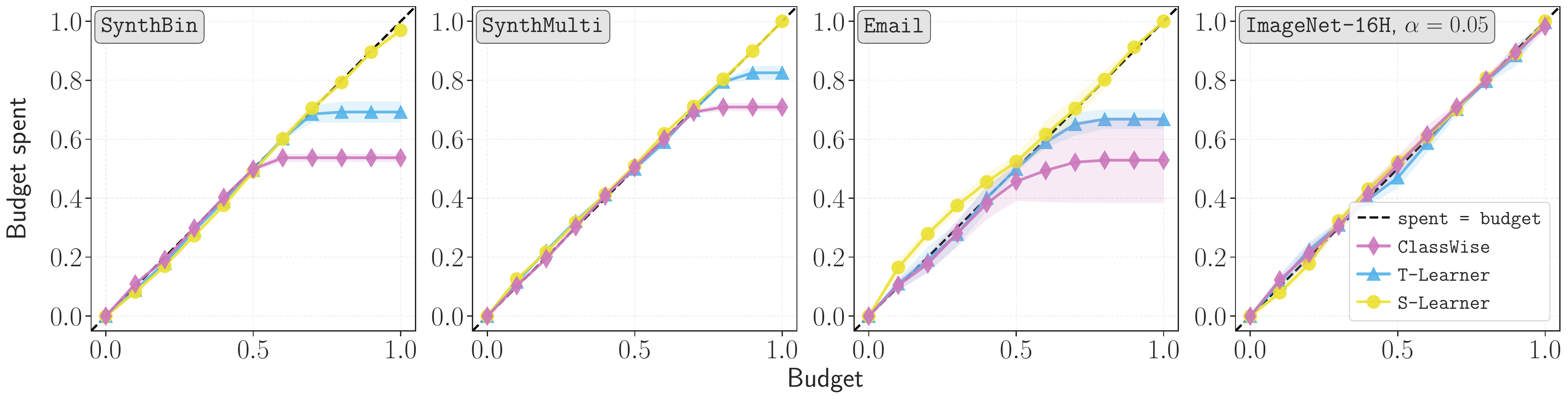}
    \caption{Fraction of instances on which the policy requests help, as a function of the nominal budget, for \SLearner{}, \TLearner{} and \ClasswiseRisk{} (mean over $5$ seeds, with $95\%$ confidence intervals). The dashed line is $\text{spent} = \text{budget}$. A curve plateaus below the diagonal when the estimated positive-VoI set is smaller than the budget, and the height of the plateau is the fraction of instances the estimator deems worth disclosing.}
    \label{fig:cp_freq_budget}
\end{figure}

\subsection{Q2: Does our learned policy improve human-AI team performance?}
\label[appendix]{app:q2}

\textbf{Take-home message}: The learned policy matches or outperforms self-selection and is associated with greater confidence discrimination between correct and incorrect classifications. Furthermore, we note that human-selected disclosure only partially align with \ClasswiseRisk{}-selected ones. 

In the main paper, we assess differences in overall accuracy across the experimental conditions. Here, in addition to reporting the full results of that analysis, we examine participants' accuracy as a function of the correctness of the ML advice provided as auxiliary information, providing a more detailed picture of how differences in adherence may contribute to differences in overall accuracy, particularly in the \Email{} task. Furthermore, we assess the effects of \textit{Support} and \textit{Budget} on participants' confidence in their own classifications, exploring the potential benefits of a selective disclosure approach beyond classification accuracy alone. Finally, to provide a more complete comparison between \ClasswiseRisk{}-selected and human-selected disclosure, we examine alignment between \ClasswiseRisk{}-selected and human-selected disclosure.

\paragraph{Overall accuracy.} We assess whether overall classification accuracy differs across the four experimental conditions by fitting a logistic mixed-effects model predicting the correctness of participants' classifications from \textit{Support}, \textit{Budget}, and their interaction. This analysis was presented in the main text and will not be further discussed, but see \cref{tab:overall-accuracy} for the full results.

\begin{table}[!ht]
\centering
~\footnotesize
\setlength{\tabcolsep}{4pt}
\renewcommand{\arraystretch}{1.0}

\begin{tabular}{llc@{\hspace{10pt}}c}
\toprule
\multicolumn{4}{c}{\textbf{Descriptive statistics}} \\
\midrule
&
&
\multicolumn{1}{c}{\textit{\Image{}}}
&
\multicolumn{1}{c}{\textit{\Email{}}} \\
\cmidrule(lr){3-3}
\cmidrule(lr){4-4}

\multicolumn{1}{c}{Support condition}
&
\multicolumn{1}{c}{Budget condition}
&
\multicolumn{1}{c}{Accuracy}
&
\multicolumn{1}{c}{Accuracy} \\
\midrule

Human Policy
& Low Budget
& $0.73 \pm 0.15$
& $0.82 \pm 0.13$ \\

Human Policy
& High Budget
& $0.81 \pm 0.11$
& $0.84 \pm 0.13$ \\

Machine Policy
& Low Budget
& $0.75 \pm 0.11$
& $0.81 \pm 0.12$ \\

Machine Policy
& High Budget
& $0.84 \pm 0.09$
& $0.84 \pm 0.12$ \\

\bottomrule
\end{tabular}

\vspace{0.75em}

\begin{tabular}{lcccc@{\hspace{10pt}}cccc}
\toprule
\multicolumn{9}{c}{\textbf{Regression coefficients}} \\
\midrule
&
\multicolumn{4}{c}{\textit{\Image{}}}
&
\multicolumn{4}{c}{\textit{\Email{}}} \\
\cmidrule(lr){2-5}
\cmidrule(lr){6-9}

\multicolumn{1}{c}{Fixed effect}
&
\multicolumn{1}{c}{OR}
&
\multicolumn{1}{c}{95\% CI}
&
\multicolumn{1}{c}{\textit{z}}
&
\multicolumn{1}{c}{\textit{p}}
&
\multicolumn{1}{c}{OR}
&
\multicolumn{1}{c}{95\% CI}
&
\multicolumn{1}{c}{\textit{z}}
&
\multicolumn{1}{c}{\textit{p}} \\
\midrule

Intercept
& 10.99
& [7.72, 15.64]
& 13.31
& $< .001$
& 7.66
& [6.24, 9.39]
& 19.50
& $< .001$ \\

Support (HP)
& 0.76
& [0.60, 0.97]
& -2.25
& .025
& 1.04
& [0.82, 1.31]
& 0.30
& .762 \\

Budget (LB)
& 0.51
& [0.40, 0.65]
& -5.45
& $< .001$
& 0.79
& [0.63, 1.01]
& -1.91
& .056 \\

Support $\times$ Budget
& 1.15
& [0.71, 1.87]
& 0.57
& .566
& 1.04
& [0.65, 1.66]
& 0.15
& .880 \\

\bottomrule
\end{tabular}

\caption{Descriptive statistics (means and standard deviations) and results of the logistic mixed-effects regressions predicting classification accuracy from \textit{Support}, \textit{Budget}, and their interaction (OR: odds ratio).}
\label{tab:overall-accuracy}

\end{table}

\paragraph{Classification accuracy by correctness of support information.}
We explore how participants' accuracy varies as a function of the correctness of the support information provided. In the \Email{} task, correctness was determined by whether the suggested class corresponded to the ground-truth label; in the \Image{} task, it was determined by whether the ground-truth label was included in the prediction set. We focus on trials in which support information was available during classification and, to improve comparability between the MP and HP conditions, consider only items for which the policy would provide support information. Because the ML model prediction constituting the support information was generally accurate, trials in which it was incorrect were relatively rare ($75/1,227$ observations in the \Email{} study and $86/1,827$ in the \Image{} study). The subgroups pertaining to wrong-advice instances thus present few observations (as evident from the large error bars characterizing relative to such subgroups in \cref{fig:mod_ans_correct}), making inferential statistics potentially unreliable; we therefore limit our discussion to the descriptive patterns.

As shown in \cref{fig:mod_ans_correct}, the two studies exhibit different patterns. In the \Image{} study, accuracy in MP relative to HP tends to be higher when the support information is correct and approximately comparable when it is incorrect. In the \Email{} study, the opposite descriptive pattern emerges: accuracy in MP is, on average, lower than in HP when the support information is correct and higher when it is incorrect. These patterns broadly correspond to the differences in adherence between MP and HP observed across the two studies (discussed in the main text and in the following section), with the relative adherence on policy-provided information being higher in the \Image{} study than in the \Email{} study.

Furthermore, this pattern may help explain why a significant difference in accuracy between MP and HP was observed in the \Image{} study but not in the \Email{} study. Because the support information was correct on most trials, even a modest reduction in participants' tendency to follow it could result in an appreciable decrease in overall accuracy.

\begin{figure}
    \includegraphics[width=\linewidth]{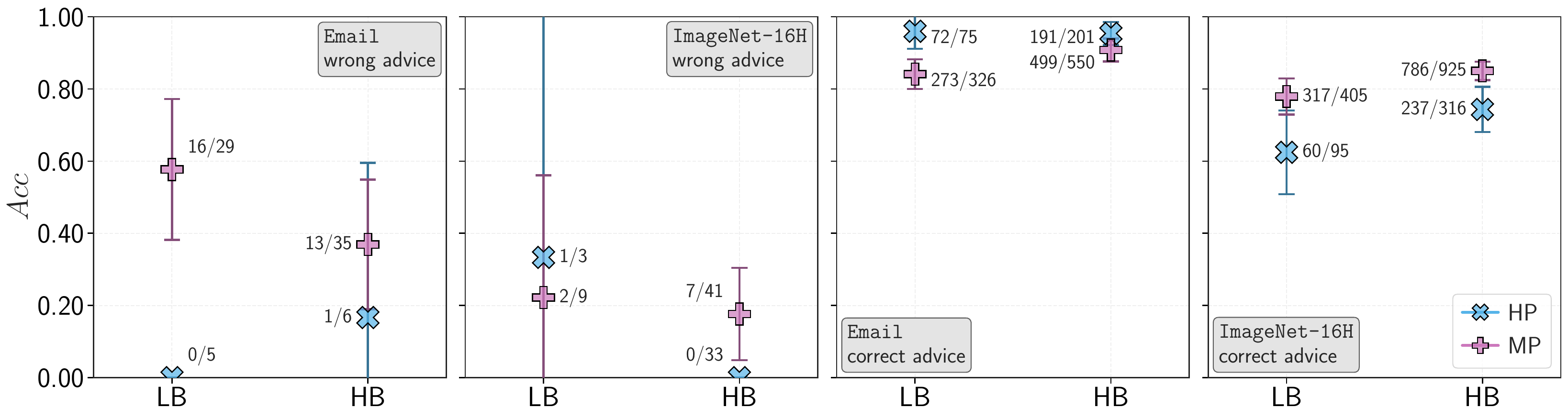}
    \caption{Participants' average accuracy by experimental condition and correctness of support information. Ratios indicate the number of correct participant responses out of the total number of observations in each subgroup. }
    \label{fig:mod_ans_correct}
\end{figure}

\paragraph{Human confidence.}
We explore participants' confidence in the correctness of their final classifications using a linear mixed-effects model in which confidence on each trial is predicted by \textit{Support}, \textit{Budget}, the presence of support information, participants' classification accuracy, and all their interactions (see \cref{tab:confidence} for the full results). This analysis allows us to examine not only whether confidence differs across experimental conditions, but also whether such differences depend on whether participants' classifications are correct. In particular, greater confidence for correct than incorrect classifications can be interpreted as greater confidence discrimination.

In the \Image{} study, we observe a significant interaction between \textit{Support}, presence of support information, and participants' accuracy ($p = .031$). When support information was not received, the difference in confidence between correctly and incorrectly classified items was significantly greater in MP than in HP (MP: $2.99 \pm 1.57$; HP: $1.91 \pm 1.29$; $p < .001$). When support information was received, the corresponding difference was descriptively greater in MP than in HP, although this contrast did not reach significance after correction for multiple comparisons (MP: $2.06 \pm 1.41$; HP: $1.50 \pm 1.23$; $p = .088$).

We also observe a significant interaction between \textit{Budget}, presence of support information, and participants' accuracy ($p = .003$). When support information was received, the difference in confidence between correct and incorrect classifications was significantly greater in HB than in LB (HB: $1.95 \pm 1.37$; LB: $1.61 \pm 1.32$; $p = .043$). In contrast, when support information was not received, the corresponding difference did not significantly differ between HB and LB (HB: $2.09 \pm 1.62$; LB: $2.62 \pm 1.42$; $p = .110$). One possible explanation is that, under the higher budget, support information was provided for a broader range of items, potentially including easier ones.

Instead, in the \Email{} study, the only significant interaction involving participants' accuracy (and thus directly relevant to confidence discrimination) was the interaction with \textit{Support} ($p = .005$). Specifically, the difference in confidence between correctly and incorrectly classified items was significantly greater in MP than in HP (MP: $0.71 \pm 0.86$; HP: $0.43 \pm 0.89$; $p = .005$).

Overall, across the two studies, participants in MP tended to show greater confidence discrimination than those in HP, assigning relatively higher confidence to correct than to incorrect classifications. Notably, this pattern was also observed in the \Email{} study, in which the selective disclosure method did not significantly increase overall classification accuracy.

\begin{table}[!ht]
\centering
~\footnotesize
\setlength{\tabcolsep}{2.5pt}
\renewcommand{\arraystretch}{0.95}

\begin{tabular}{lrrrr@{\hspace{10pt}}rrrr}
\toprule
\multicolumn{7}{c}{\textbf{Descriptive statistics}} \\
\midrule
&
&
&
\multicolumn{2}{c}{\textit{\Image{}}}
&
\multicolumn{2}{c}{\textit{\Email{}}} \\
\cmidrule(lr){4-5}
\cmidrule(lr){6-7}

\multicolumn{1}{c}{Support}
&
\multicolumn{1}{c}{Budget}
&
\multicolumn{1}{c}{Information}
&
\multicolumn{1}{c}{Incorrect}
&
\multicolumn{1}{c}{Correct}
&
\multicolumn{1}{c}{Incorrect}
&
\multicolumn{1}{c}{Correct} \\
\midrule

Human Policy
& Low Budget
& Not received
& $4.22 \pm 1.45$
& $6.30 \pm 0.67$
& $5.20 \pm 1.29$
& $5.56 \pm 0.99$ \\

Human Policy
& Low Budget
& Received
& $3.24 \pm 1.49$
& $4.58 \pm 1.45$
& $5.32 \pm 1.30$
& $5.67 \pm 0.98$ \\

Human Policy
& High Budget
& Not received
& $4.74 \pm 1.60$
& $6.45 \pm 0.50$
& $5.24 \pm 1.28$
& $5.69 \pm 0.89$ \\

Human Policy
& High Budget
& Received
& $3.08 \pm 1.46$
& $4.75 \pm 1.48$
& $5.00 \pm 1.31$
& $5.42 \pm 1.19$ \\

Machine Policy
& Low Budget
& Not received
& $3.11 \pm 1.45$
& $6.22 \pm 0.67$
& $4.82 \pm 1.28$
& $5.59 \pm 0.87$ \\

Machine Policy
& Low Budget
& Received
& $3.68 \pm 1.68$
& $5.58 \pm 1.08$
& $5.11 \pm 1.26$
& $5.52 \pm 0.93$ \\

Machine Policy
& High Budget
& Not received
& $3.84 \pm 1.96$
& $6.63 \pm 0.42$
& $4.74 \pm 1.16$
& $5.52 \pm 0.78$ \\

Machine Policy
& High Budget
& Received
& $3.61 \pm 1.59$
& $5.79 \pm 0.87$
& $4.77 \pm 1.36$
& $5.62 \pm 0.95$ \\

\bottomrule
\end{tabular}

\vspace{0.5em}

\begin{tabular}{lrrrr@{\hspace{10pt}}rrrr}
\toprule
\multicolumn{9}{c}{\textbf{Regression coefficients}} \\
\midrule

&
\multicolumn{4}{c}{\textit{\Image{}}}
&
\multicolumn{4}{c}{\textit{\Email{}}} \\
\cmidrule(lr){2-5}
\cmidrule(lr){6-9}

\multicolumn{1}{c}{Fixed effect}
&
\multicolumn{1}{c}{$b$}
& \multicolumn{1}{c}{SE}
& \multicolumn{1}{c}{$t$}
& \multicolumn{1}{c}{$p$}
&
\multicolumn{1}{c}{$b$}
& \multicolumn{1}{c}{SE}
& \multicolumn{1}{c}{$t$}
& \multicolumn{1}{c}{$p$} \\
\midrule

Intercept
& 4.504 & .096 & 46.79 & $<.001$
& 5.168 & .069 & 75.19 & $<.001$ \\

Support (HP)
& .732 & .150 & 4.89 & $<.001$
& .440 & .134 & 3.30 & .001 \\

Budget (LB)
& -.383 & .150 & -2.56 & .011
& .064 & .133 & .48 & .632 \\

Information received
& -.062 & .093 & -.67 & .503
& -.099 & .082 & -1.22 & .223 \\

Accuracy
& 1.480 & .075 & 19.81 & $<.001$
& .410 & .047 & 8.66 & $<.001$ \\

Support $\times$ Budget
& .192 & .297 & .65 & .518
& -.134 & .266 & -.50 & .615 \\

Support $\times$ Information
& -.989 & .178 & -5.56 & $<.001$
& -.403 & .165 & -2.44 & .015 \\

Budget $\times$ Information
& .544 & .175 & 3.10 & .002
& .138 & .161 & .86 & .391 \\

Support $\times$ Accuracy
& -.716 & .136 & -5.27 & $<.001$
& -.362 & .093 & -3.91 & $<.001$ \\

Budget $\times$ Accuracy
& .261 & .136 & 1.92 & .055
& -.087 & .092 & -.95 & .344 \\

Information $\times$ Accuracy
& -.438 & .103 & -4.24 & $<.001$
& .083 & .089 & .94 & .349 \\

Support $\times$ Budget $\times$ Information
& -.122 & .348 & -.35 & .727
& .449 & .320 & 1.40 & .160 \\

Support $\times$ Budget $\times$ Accuracy
& -.049 & .270 & -.18 & .856
& -.066 & .184 & -.36 & .721 \\

Support $\times$ Information $\times$ Accuracy
& .428 & .199 & 2.15 & .031
& .225 & .178 & 1.26 & .207 \\

Budget $\times$ Information $\times$ Accuracy
& -.587 & .197 & -2.98 & .003
& -.065 & .176 & -.37 & .710 \\

Support $\times$ Budget $\times$ Information $\times$ Accuracy
& -.180 & .390 & -.46 & .646
& -.032 & .350 & -.09 & .928 \\

\bottomrule
\end{tabular}

\vspace{0.5em}

\begin{tabular}{lrrrr@{\hspace{10pt}}rrrr}
\toprule
\multicolumn{9}{c}{\textbf{Post-hoc contrasts}} \\
\midrule

&
\multicolumn{4}{c}{\textit{\Image{}}}
&
\multicolumn{4}{c}{\textit{\Email{}}} \\
\cmidrule(lr){2-5}
\cmidrule(lr){6-9}

\multicolumn{1}{c}{Contrast}
&
\multicolumn{1}{c}{Estimate}
& \multicolumn{1}{c}{SE}
& \multicolumn{1}{c}{$z$}
& \multicolumn{1}{c}{$p$}
&
\multicolumn{1}{c}{Estimate}
& \multicolumn{1}{c}{SE}
& \multicolumn{1}{c}{$z$}
& \multicolumn{1}{c}{$p$} \\
\midrule

HP vs.\ MP, information not received
& -.716 & .136 & -5.27 & $<.001$
& \multicolumn{4}{c}{---} \\

HP vs.\ MP, information received
& -.288 & .143 & -2.02 & .088
& \multicolumn{4}{c}{---} \\

LB vs.\ HB, information not received
& .261 & .136 & 1.92 & .110
& \multicolumn{4}{c}{---} \\

LB vs.\ HB, information received
& -.326 & .142 & -2.30 & .043
& \multicolumn{4}{c}{---} \\

HP vs.\ MP
& \multicolumn{4}{c}{---}
& -.249 & .089 & -2.79 & .005 \\

\bottomrule
\end{tabular}

\caption{Descriptive statistics and results of the linear mixed-effects models predicting confidence ratings from \textit{Support}, \textit{Budget}, presence of information, classification accuracy, and their interactions. Descriptive statistics represent participant-level means $\pm$ standard deviations. Post-hoc contrasts compare confidence discrimination, defined as the difference in confidence between correct and incorrect classifications (post-hoc contrasts are reported only for significant effects involving classification accuracy; dashes indicate that no post-hoc contrast was conducted because the corresponding effect was not significant)}.
\label{tab:confidence}
\end{table}

\paragraph{Alignment between \ClasswiseRisk{}-selected and human-selected disclosure.} We explore alignment between \ClasswiseRisk{}-selected and human-selected information requests from two perspectives. First, we analyze differences between the frequency of information disclosure by the policy and humans using one-sample $t$-tests. Participants in the HP conditions request auxiliary information on fewer trials than the method would provide it in both the \Image{} and \Email{} tasks (see \cref{tab:help-used}).

Second, focusing on trials in which HP participants requested information, we descriptively examine how often these requests correspond to items for which \ClasswiseRisk{} would also have provided information. In the \Image{} task, the proportion of human-requested trials that were also selected by \ClasswiseRisk{} was $0.40 \pm 0.30$ in LB and $0.87 \pm 0.13$ in HB. In the \Email{} task, the corresponding proportions were $0.31 \pm 0.26$ and $0.56 \pm 0.20$, respectively. These proportions indicate that human requests only partially overlap with policy-selected disclosure, although the extent of this overlap differs considerably across tasks and budget conditions.

Taken together, these analyses indicate that humans and \ClasswiseRisk{} allocate disclosure differently. Participants request information on fewer trials than the policy provides it, which may limit the potential benefits of human-selected disclosure when useful assistance remains unrequested. Determining why participants request assistance relatively infrequently (e.g., because of overconfidence in their own classification ability or uncertainty about the value of assistance) requires further study, as the present user studies are designed to compare \ClasswiseRisk{}-selected and human-selected disclosure at the level of the complete strategies rather than identify the mechanisms underlying human request behavior. In addition, the descriptive overlap analysis indicates that the two strategies do not necessarily select the same trials for disclosure. These exploratory results therefore suggest that differences between policy-selected and human-selected disclosure concern both how frequently information is requested and which items receive it, motivating future work on the consequences of these differences for human-AI team performance.

\begin{table}[!ht]
\centering
~\footnotesize
\setlength{\tabcolsep}{4pt}
\renewcommand{\arraystretch}{1.0}

\begin{tabular}{llcccccc}
\toprule
Study
& Budget
& Requested assistance
& MP assistance
& 95\% CI
& \textit{t}
& \textit{df}
& \textit{p} \\
\midrule

\Image{}
& Low
& $3.73 \pm 2.12$
& 6
& [3.21, 4.25]
& $-8.70$
& 65
& $< .001$ \\

\Image{}
& High
& $5.99 \pm 3.20$
& 14
& [5.22, 6.75]
& $-20.82$
& 68
& $< .001$ \\

\Email{}
& Low
& $3.93 \pm 1.97$
& 5
& [3.47, 4.39]
& $-4.60$
& 71
& $< .001$ \\

\Email{}
& High
& $6.41 \pm 4.63$
& 9
& [5.25, 7.56]
& $-4.48$
& 63
& $< .001$ \\

\bottomrule
\end{tabular}

\caption{Number of assistance requests made by participants in the Human-Policy condition compared with the number of trials on which assistance was provided in the corresponding Machine-Policy condition. Requested assistance is reported as mean $\pm$ standard deviation. Confidence intervals and \textit{t} tests refer to one-sample tests comparing the observed number of requests with the number of assistance opportunities provided in the corresponding MP condition.}
\label{tab:help-used}
\end{table}

\subsection{Q3: How do participants interact with our policy?}
\label[appendix]{app:q3}

\textbf{Take-home message}: The learned policy is more beneficial whenever humans follow the advice, and self-reported trust in the AI system is broadly consistent with advice-adherence trends.

In addition to the analyses of advice adherence discussed in the main text, for which we report full details here, we analyze participants' responses to the trust scale, providing a complementary, self-reported perspective on their trust in the system providing auxiliary information.

\paragraph{Advice adherence.}
We assess participants' adherence to the AI advice provided as support information by fitting a logistic mixed-effects model predicting whether participants' classifications are consistent with the AI advice from \textit{Support}, \textit{Budget}, and their interaction. This analysis was presented in the main text and will not be further discussed, but see \cref{tab:behavioural-trust} for the full results. ~\footnote{Because adherence is operationalized differently in the \Image{} and \Email{} tasks, due to the nature of the advice provided, we repeat the \Image{} analyses using only trials with singleton prediction sets. On these trials, the advice specifies a single label and is therefore directly comparable to that provided in the \Email{} task (moreover, singleton sets represents the majority of cases, accounting for 64\% and 65\% of the observations in the analyzed LB and HB subsets, respectively). The pattern of results remains unchanged.}

\begin{table}[!ht]
\centering
~\footnotesize
\setlength{\tabcolsep}{4pt}
\renewcommand{\arraystretch}{1.0}

\begin{tabular}{llc@{\hspace{10pt}}c}
\toprule
\multicolumn{4}{c}{\textbf{Descriptive statistics}} \\
\midrule
&
&
\multicolumn{1}{c}{\textit{\Image{}}}
&
\multicolumn{1}{c}{\textit{\Email{}}} \\
\cmidrule(lr){3-3}
\cmidrule(lr){4-4}

\multicolumn{1}{c}{Support condition}
&
\multicolumn{1}{c}{Budget condition}
&
\multicolumn{1}{c}{\shortstack{Proportion of agreement\\with support information}}
&
\multicolumn{1}{c}{\shortstack{Proportion of agreement\\with support information}} \\
\midrule

Human Policy
& Low Budget
& $0.92 \pm 0.24$
& $0.96 \pm 0.17$ \\

Human Policy
& High Budget
& $0.94 \pm 0.13$
& $0.95 \pm 0.12$ \\

Machine Policy
& Low Budget
& $0.91 \pm 0.12$
& $0.81 \pm 0.17$ \\

Machine Policy
& High Budget
& $0.92 \pm 0.10$
& $0.89 \pm 0.14$ \\

\bottomrule
\end{tabular}

\vspace{0.75em}

\begin{tabular}{lcccc@{\hspace{10pt}}cccc}
\toprule
\multicolumn{9}{c}{\textbf{Regression coefficients}} \\
\midrule
&
\multicolumn{4}{c}{\textit{\Image{}}}
&
\multicolumn{4}{c}{\textit{\Email{}}} \\
\cmidrule(lr){2-5}
\cmidrule(lr){6-9}

\multicolumn{1}{c}{Fixed effect}
&
\multicolumn{1}{c}{OR}
&
\multicolumn{1}{c}{95\% CI}
&
\multicolumn{1}{c}{\textit{z}}
&
\multicolumn{1}{c}{\textit{p}}
&
\multicolumn{1}{c}{OR}
&
\multicolumn{1}{c}{95\% CI}
&
\multicolumn{1}{c}{\textit{z}}
&
\multicolumn{1}{c}{\textit{p}} \\
\midrule

Intercept
& 48.02
& [25.31, 91.13]
& 11.85
& $< .001$
& 25.08
& [13.88, 45.30]
& 10.68
& $< .001$ \\

Support condition (HP)
& 1.72
& [0.93, 3.18]
& 1.73
& .084
& 4.63
& [2.14, 10.04]
& 3.89
& $< .001$ \\

Budget condition (LB)
& 0.70
& [0.37, 1.31]
& -1.13
& .260
& 0.83
& [0.38, 1.81]
& -0.47
& .636 \\

Support $\times$ Budget
& 0.59
& [0.18, 1.97]
& -0.86
& .392
& 3.31
& [0.72, 15.29]
& 1.53
& .126 \\

\bottomrule
\end{tabular}

\caption{Descriptive statistics and results of the logistic mixed-effects regressions predicting participants' agreement with the support information from \textit{Support}, \textit{Budget}, and their interaction (OR: odds ratio).}
\label{tab:behavioural-trust}

\end{table}

\paragraph{Trust in the AI system.} We analyze the trust index, computed by averaging participants' ratings across the eight items of the trust scale, using linear models in which trust index values are predicted by \textit{Support} condition, \textit{Budget} condition, and their interaction (see \cref{tab:trust-index} for the full results).

In the \Image{} study, trust index values do not significantly differ between MP and HP (MP: $3.06 \pm 0.73$; HP: $2.90 \pm 0.75$; $p = .064$). while they are significantly higher in HB than in LB (HB: $3.13 \pm 0.70$; LB: $2.83 \pm 0.75$; $p < .001$), while the interaction between \textit{Support} and \textit{Budget} is not significant ($p = .096$). In contrast, in the \Email{} study, trust index values are significantly lower in MP than in HP (MP: $3.04 \pm 0.82$; HP: $3.32 \pm 0.70$; $p = .003$). HB and LB do not significantly differ ($p = .861$), whereas the interaction between \textit{Support} and \textit{Budget} is significant ($p = .046$). Follow-up comparisons indicate that the difference between MP and HP is significant under LB (MP-LB: $2.95 \pm 0.81$; HP-LB: $3.40 \pm 0.70$; $p = .002$), but not under HB (MP-HB: $3.15 \pm 0.82$; HP-HB: $3.23 \pm 0.70$; $p = 1$).

This pattern is broadly consistent with the results on adherence on support information: participants in MP reported similar trust in the AI system to those in HP in the \Image{} study, whereas in the \Email{} study, particularly under the lower budget, participants in MP reported lower values than those in HP.

\begin{table}[!ht]
\centering
~\footnotesize
\setlength{\tabcolsep}{4pt}
\renewcommand{\arraystretch}{1.0}

\begin{tabular}{llc@{\hspace{10pt}}c}
\toprule
\multicolumn{4}{c}{\textbf{Descriptive statistics}} \\
\midrule
&
&
\multicolumn{1}{c}{\textit{\Image{}}}
&
\multicolumn{1}{c}{\textit{\Email{}}} \\
\cmidrule(lr){3-3}
\cmidrule(lr){4-4}

\multicolumn{1}{c}{Support condition}
&
\multicolumn{1}{c}{Budget condition}
&
\multicolumn{1}{c}{Trust index}
&
\multicolumn{1}{c}{Trust index} \\
\midrule

Human Policy
& Low Budget
& $2.67 \pm 0.72$
& $3.40 \pm 0.70$ \\

Human Policy
& High Budget
& $3.12 \pm 0.71$
& $3.23 \pm 0.70$ \\

Machine Policy
& Low Budget
& $2.98 \pm 0.75$
& $2.95 \pm 0.81$ \\

Machine Policy
& High Budget
& $3.14 \pm 0.69$
& $3.15 \pm 0.82$ \\

\bottomrule
\end{tabular}

\vspace{0.75em}

\begin{tabular}{lcccc@{\hspace{10pt}}cccc}
\toprule
\multicolumn{9}{c}{\textbf{Regression coefficients}} \\
\midrule
&
\multicolumn{4}{c}{\textit{\Image{}}}
&
\multicolumn{4}{c}{\textit{\Email{}}} \\
\cmidrule(lr){2-5}
\cmidrule(lr){6-9}

\multicolumn{1}{c}{Fixed effect}
&
\multicolumn{1}{c}{$b$}
&
\multicolumn{1}{c}{SE}
&
\multicolumn{1}{c}{\textit{t}}
&
\multicolumn{1}{c}{\textit{p}}
&
\multicolumn{1}{c}{$b$}
&
\multicolumn{1}{c}{SE}
&
\multicolumn{1}{c}{\textit{t}}
&
\multicolumn{1}{c}{\textit{p}} \\
\midrule

Intercept
& 2.978
& 0.044
& 68.37
& $< .001$
& 3.182
& 0.046
& 69.12
& $< .001$ \\

Support (HP)
& -0.162
& 0.087
& -1.86
& .064
& 0.273
& 0.092
& 2.96
& .003 \\

Budget (LB)
& -0.303
& 0.087
& -3.48
& $< .001$
& -0.016
& 0.092
& -0.18
& .861 \\

Support $\times$ Budget
& -0.291
& 0.174
& -1.67
& .096
& 0.369
& 0.184
& 2.00
& .046 \\

\bottomrule
\end{tabular}

\vspace{0.75em}

\begin{tabular}{lrrrr}
\toprule
\multicolumn{5}{c}{\textbf{Post-hoc comparisons}: \Email{}} \\
\midrule
Contrast
& Estimate
& SE
& \textit{t}
& \textit{p} \\
\midrule

HP-LB vs.\ MP-LB
& 0.457
& 0.127
& 3.61
& .002 \\

HP-LB vs.\ HP-HB
& 0.168
& 0.130
& 1.29
& 1 \\

HP-LB vs.\ MP-HB
& 0.257
& 0.130
& 1.98
& .294 \\

MP-LB vs.\ HP-HB
& -0.289
& 0.131
& -2.21
& .167 \\

MP-LB vs.\ MP-HB
& -0.201
& 0.130
& -1.54
& .745 \\

HP-HB vs.\ MP-HB
& 0.088
& 0.134
& 0.66
& 1 \\

\bottomrule
\end{tabular}

\caption{Descriptive statistics and results of the linear models predicting trust index values from \textit{Support}, \textit{Budget}, and their interaction. Descriptive statistics represent means $\pm$ standard deviations. Post-hoc comparisons for the \Email{} study are Bonferroni-corrected for six comparisons.}
\label{tab:trust-index}

\end{table}

\paragraph{Counterfactual benchmark.} We report here the results of the counterfactual analysis, in which we assess the performance that MP participants would have achieved had their predictions been overruled with the ML advice when provided (for the \Image{} task, we convert each prediction set into a single prediction by retaining its highest-scoring class, which is treated as the counterfactual response). This analysis was presented in the main text and will not be further discussed, but see \cref{tab:counterfactual-reliance} for the full results. ~\footnote{Also in this case, restricting the \Image{} analyses to trials in which the AI advice consists of a singleton prediction set leaves the pattern of results broadly unchanged. The only difference is that the contrast between MP and MPCFT becomes significant, favoring the latter especially in HB.}

\begin{table}[!ht]
\centering
~\footnotesize
\setlength{\tabcolsep}{4pt}
\renewcommand{\arraystretch}{1.0}

\begin{tabular}{lc@{\hspace{10pt}}c}
\toprule
\multicolumn{3}{c}{\textbf{Descriptive statistics}} \\
\midrule
&
\multicolumn{1}{c}{\textit{\Image{}}}
&
\multicolumn{1}{c}{\textit{\Email{}}} \\
\cmidrule(lr){2-2}
\cmidrule(lr){3-3}

\multicolumn{1}{c}{Support condition}
&
\multicolumn{1}{c}{Accuracy}
&
\multicolumn{1}{c}{Accuracy} \\
\midrule

Human Policy
& $0.77 \pm 0.13$
& $0.83 \pm 0.13$ \\

Machine Policy
& $0.80 \pm 0.11$
& $0.82 \pm 0.12$ \\

Machine Policy -- Counterfactual benchmark
& $0.82 \pm 0.11$
& $0.85 \pm 0.10$ \\

\bottomrule
\end{tabular}

\vspace{0.75em}

\begin{tabular}{lcccc@{\hspace{10pt}}cccc}
\toprule
\multicolumn{9}{c}{\textbf{Regression coefficients}} \\
\midrule
&
\multicolumn{4}{c}{\textit{\Image{}}}
&
\multicolumn{4}{c}{\textit{\Email{}}} \\
\cmidrule(lr){2-5}
\cmidrule(lr){6-9}

\multicolumn{1}{c}{Fixed effect}
&
\multicolumn{1}{c}{OR}
&
\multicolumn{1}{c}{95\% CI}
&
\multicolumn{1}{c}{\textit{z}}
&
\multicolumn{1}{c}{\textit{p}}
&
\multicolumn{1}{c}{OR}
&
\multicolumn{1}{c}{95\% CI}
&
\multicolumn{1}{c}{\textit{z}}
&
\multicolumn{1}{c}{\textit{p}} \\
\midrule

Intercept
& 14.40
& [10.01, 20.71]
& 14.38
& $< .001$
& 9.18
& [7.35, 11.48]
& 19.49
& $< .001$ \\

Support - HP
& 0.59
& [0.44, 0.80]
& -3.40
& $< .001$
& 0.87
& [0.63, 1.19]
& -0.87
& .385 \\

Support - MP
& 1.05
& [0.84, 1.32]
& 0.43
& .668
& 0.82
& [0.66, 1.03]
& -1.71
& .087 \\

Budget (LB)
& 0.44
& [0.35, 0.56]
& -6.90
& $< .001$
& 0.76
& [0.59, 0.97]
& -2.23
& .026 \\

Support - HP $\times$ Budget
& 1.41
& [0.77, 2.57]
& 1.12
& .264
& 1.12
& [0.59, 2.12]
& 0.35
& .727 \\

Support - MP $\times$ Budget
& 1.07
& [0.68, 1.70]
& 0.31
& .758
& 1.00
& [0.64, 1.57]
& 0.02
& .983 \\

\bottomrule
\end{tabular}

\vspace{0.75em}

\begin{tabular}{lcccc@{\hspace{10pt}}cccc}
\toprule
\multicolumn{9}{c}{\textbf{Post-hoc comparisons}} \\
\midrule
&
\multicolumn{4}{c}{\textit{\Image{}}}
&
\multicolumn{4}{c}{\textit{\Email{}}} \\
\cmidrule(lr){2-5}
\cmidrule(lr){6-9}

\multicolumn{1}{c}{Contrast}
&
\multicolumn{1}{c}{OR}
&
\multicolumn{1}{c}{SE}
&
\multicolumn{1}{c}{\textit{z}}
&
\multicolumn{1}{c}{\textit{p}}
&
\multicolumn{1}{c}{OR}
&
\multicolumn{1}{c}{SE}
&
\multicolumn{1}{c}{\textit{z}}
&
\multicolumn{1}{c}{\textit{p}} \\
\midrule

HP vs.\ MP
& 0.75
& 0.09
& -2.33
& .059
& 1.03
& 0.13
& 0.21
& 1 \\

HP vs.\ MPCFT
& 0.61
& 0.08
& -4.00
& $< .001$
& 0.79
& 0.10
& -1.84
& .196 \\

MP vs.\ MPCFT
& 0.81
& 0.07
& -2.34
& .058
& 0.77
& 0.06
& -3.25
& .004 \\

\bottomrule
\end{tabular}

\caption{Descriptive statistics and results of the logistic mixed-effects
regressions comparing observed accuracy in HP and MP with counterfactual accuracy under full advice adherence on
automatically disclosed support information (MPCFT).
Descriptive statistics represent participant-level means $\pm$ standard
deviations. Post-hoc comparisons are averaged across budget conditions and
Bonferroni-corrected for three comparisons (\emph{OR}: odds ratio).}
\label{tab:counterfactual-reliance}

\end{table}

\subsection{Q4: How does changing the information content affect our policy?}
\label[appendix]{app:alpha}

\textbf{Take-home message}: The benefit of selective disclosure depends on how informative the support information is. At $\alpha = 0.05$, where prediction sets are small and informative, \ClasswiseRisk{} prioritizes them at every budget; as $\alpha$ decreases and uninformative full sets become frequent, the gain of any disclosure policy shrinks, and at $\alpha = 0.01$ no policy separates from full disclosure.

In the \Image{} setting the support information $S$ is the conformal prediction set $C(\mbx)$, whose size is governed by the miscoverage level $\alpha$: smaller $\alpha$ enforces higher coverage and therefore yields larger sets. The size of $C(\mbx)$ directly controls how informative disclosure can be. In particular, a set that coincides with the full label space, $|C(\mbx)| = |\mcY| = 16$, is compatible with every class and thus carries \emph{no} discriminative signal: it cannot help the decision-maker refine their judgement, so its Value of Information is non-positive and a well-behaved policy should never spend budget disclosing it. The prevalence of such full sets consequently upper-bounds the benefit attainable at a given $\alpha$.

We examine the role of $\alpha$ along three axes: (i) we analyse how the distribution of set sizes varies with $\alpha$, motivating the choice of $\alpha=0.05$ in the main paper; (ii) we report the budget--accuracy curves for the remaining levels $\alpha\in\{0.01,0.02\}$ and comment on how the achievable benefit shrinks as sets grow; and (iii) we inspect \emph{which} sets our policy actually chooses to disclose as a function of the budget, showing that at $\alpha=0.05$ it essentially never wastes budget on uninformative full sets.

\paragraph{(i) Set-size distribution and the choice of \texorpdfstring{$\alpha$}{alpha}.}
\begin{table}
\centering
\caption{Distribution of conformal prediction set sizes on the \Image{} dataset ($n=1200$ images, $|\mcY|=16$ classes) for the three miscoverage levels $\alpha$ considered in our experiments. Smaller $\alpha$ enforces higher coverage and thus yields larger sets; note in particular the mass on the full set ($|C(\mbx)|=16$), which carries no information for the decision-maker. Dashes denote empty bins.}
\label{tab:cp_set_sizes}
\setlength{\tabcolsep}{4pt}
\resizebox{\textwidth}{!}{%
\begin{tabular}{@{}lrrrrrrrrrrrrrrrr rr@{}}
\toprule
& \multicolumn{16}{c}{Set size $|C(\mbx)|$} & \multicolumn{2}{c}{Summary} \\
\cmidrule(lr){2-17}\cmidrule(l){18-19}
$\alpha$ & $1$ & $2$ & $3$ & $4$ & $5$ & $6$ & $7$ & $8$ & $9$ & $10$ & $11$ & $12$ & $13$ & $14$ & $15$ & $16$ & Mean & Sing.\% \\
\midrule
$0.01$ & 664 &  92 & 32 &  9 & 10 &  3 & 1 & 1 & -- &  1 & -- & -- & -- & -- & -- & 387 & $6.05$ & $55.3$ \\
$0.02$ & 735 & 148 & 54 & 31 & 16 & 13 & 9 & 8 &  4 &  4 &  2 &  1 & -- &  5 &  1 & 169 & $3.75$ & $61.3$ \\
$0.05$ & 804 & 200 & 99 & 55 & 24 & 13 & 3 & -- & -- & -- & -- & -- & -- & -- & -- &   2 & $1.64$ & $67.0$ \\
\bottomrule
\end{tabular}%
}
\end{table}
\cref{tab:cp_set_sizes} reports the distribution of set sizes on the \Image{} dataset for the three miscoverage levels considered in our experiments. As expected, the smaller $\alpha$, the larger the prediction sets: the mean set size shrinks from $6.05$ at $\alpha=0.01$ to $1.64$ at $\alpha=0.05$. The focus for our purposes is not the exact size but whether the disclosed set is \emph{informative}: a large set  provides little guidance to the decision-maker, and in the limit the full set ($|C(\mbx)|=16$) carries no signal at all. From this perspective the three levels differ sharply. 

At $\alpha=0.05$ the support information is almost always useful: essentially all sets have size below $7$, and only $2$ out of $1200$ images fall on large, uninformative sets ($|C(\mbx)|>10$). 
At $\alpha=0.01$ and $\alpha=0.02$, instead, a substantial fraction of the sets is large and uninformative (the full set alone accounts for $387$ and $169$ images respectively) so a non-negligible share of any disclosed feedback would be effectively useless. This is precisely why we adopt $\alpha=0.05$ in the main paper: it is the level at which disclosure is almost always informative, so that revealing a set genuinely helps the human decision rather than wasting budget on redundant, uninformative feedback.

\paragraph{(ii) Disclosure accuracy across miscoverage levels.}
\cref{fig:cp_alpha_acc} reports the budget--accuracy curves for $\alpha \in \{0.05, 0.02, 0.01\}$, complementing the results in the main paper (\cref{fig:results}). A first effect of lowering $\alpha$ is visible already at the endpoints: as the prediction sets grow larger and less discriminative, the average benefit of full disclosure shrinks. Full disclosure (\hone) attains about $80.5\%$ at both $\alpha = 0.01$ and $\alpha = 0.02$, i.e.\ only about $3.3$ points above the no-disclosure accuracy (\hzero $\approx 77.2\%$), against the $8.3$-point gap observed at $\alpha = 0.05$. Larger sets thus carry less usable signal, and the ceiling that any disclosure policy can reach is correspondingly lower. 

At $\alpha = 0.02$, selective disclosure remains effective: \ClasswiseRisk{} reaches the accuracy of full disclosure already at $B = 0.5$, whereas \Confidence{} needs $B = 0.8$ and \Random{} the full budget. Both policies then slightly exceed \hone{} at higher budgets (\ClasswiseRisk{} $\approx 81.0\%$ at $B = 0.7$, \Confidence{} $\approx 81.3\%$ at $B = 0.9$), although these gains are within the confidence bands.
At $\alpha = 0.01$, instead, no policy separates from full disclosure: \ClasswiseRisk{}, \Confidence{} and even \Random{} track \hone{} within their confidence bands, and \ClasswiseRisk{} falls slightly below it at high budgets ($\approx 79.3\%$ at $B = 0.9$, about three images).

\begin{figure}[t]
    \centering
    \includegraphics[width=\linewidth]{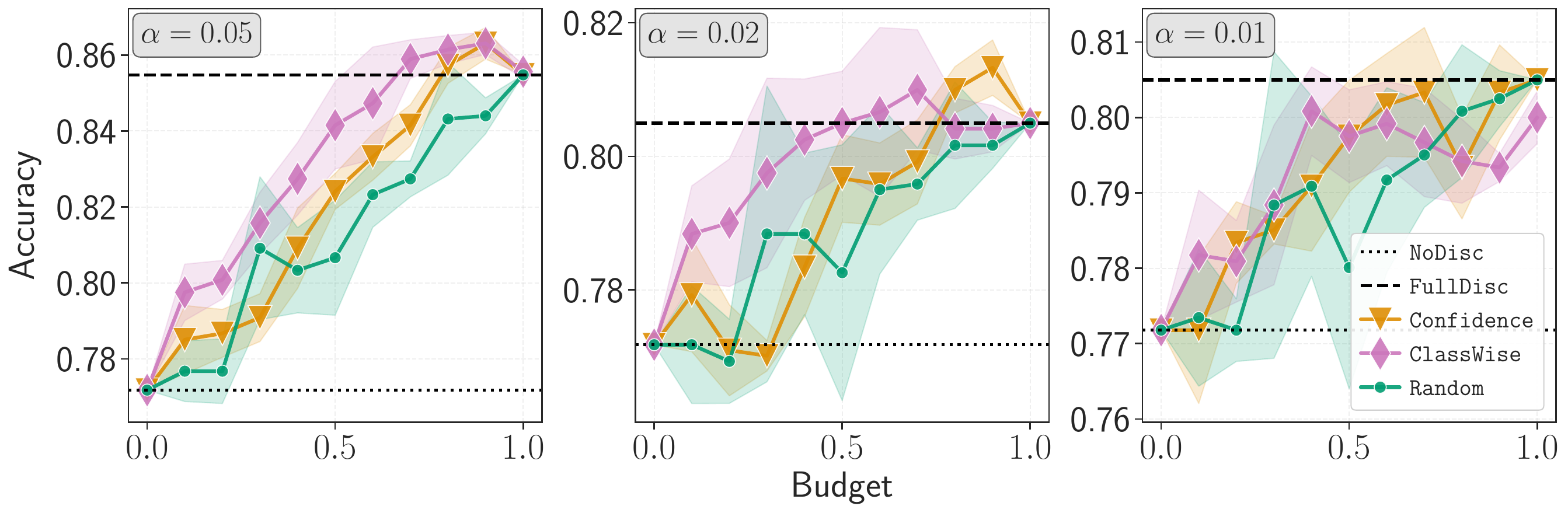}
    \caption{Budget--accuracy curves on the \Image{} dataset for the two additional miscoverage levels $\alpha\in\{0.01,0.02\}$. As $\alpha$ decreases the prediction sets grow larger and less informative, lowering the average benefit of full disclosure (\hone$\approx80.5\%$ against $85.5\%$ at $\alpha=0.05$).}
    \label{fig:cp_alpha_acc}
\end{figure}

We attribute this behaviour to the set-size distribution at $\alpha = 0.01$ (\cref{tab:cp_set_sizes}). When $\alpha$ is very small, the prediction sets are large: the mean set size is $6.05$, and about one third of the images ($387/1200$) receive the full, completely uninformative set. First, fewer instances receive an informative set, so the mass of instances with positive VoI is smaller and a selective policy has less signal to exploit. Second, estimating the VoI becomes harder: the risk models must tell apart many large, overlapping sets whose effect on the human decision is weak and noisy. Ranking errors then cause the policy to withhold informative sets in favour of uninformative ones, which explains why it can fall slightly below full disclosure at high budgets. Overall, these results indicate that the benefit of selective disclosure depends on how informative the support information is: when it carries little signal, as at $\alpha = 0.01$, there is little for any policy to gain over simply disclosing everything.

\paragraph{(iii) Does the policy disclose informative sets?}
Since a large prediction set carries little signal, a good disclosure policy should spend its budget on small, informative sets and avoid the uninformative full set ($|C(\mbx)| = 16$) whenever possible. Here we ask whether the \ClasswiseRisk{} policy does so, and how this behaviour changes with the budget. \cref{fig:cp_size} shows, for each budget $B$, the fraction of test instances of each set size to which the policy discloses the support information: a policy indifferent to set size would disclose a fraction $B$ of the instances of every size.

At $\alpha = 0.05$, the policy clearly favours small sets at every budget. Full sets are essentially absent, since only $2$ exist in the whole dataset.
At $\alpha = 0.02$ and $\alpha = 0.01$, the picture becomes budget-dependent. At small budgets the policy still prioritizes small sets, showing that the estimated VoI ranks informative sets first. As the budget grows, however, the policy increasingly discloses full sets as well, reaching on average $77$ of the $85$ full sets in the test set ($91\%$) at $B = 1$ for $\alpha = 0.01$, and all $38$ of the $38$ ($100\%$) full sets for $\alpha = 0.02$.

These results confirm $\alpha=0.05$ as the most favourable operating point: the regime in which the policy can disclose informative feedback at \emph{every} budget level. These motivate its use in the main paper.

\begin{figure}
    \centering
    \includegraphics[width=\linewidth]{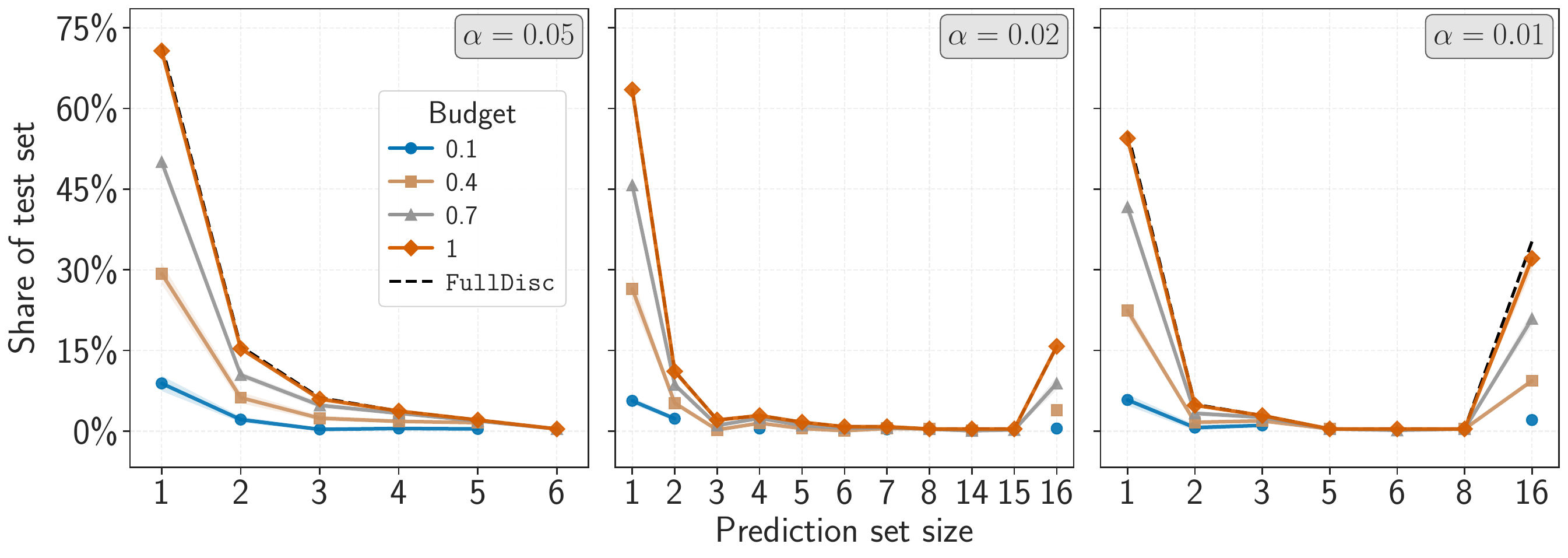}
    \caption{Prediction-set sizes disclosed by the \ClasswiseRisk{} policy on \Image{}, at $\alpha \in \{0.05, 0.02, 0.01\}$. For each set size, the curves show the share of the test set that the policy discloses at budgets $B \in \{0.1, 0.4, 0.7\}$ (mean over $5$ seeds, with $95\%$ confidence intervals). The dashed line (\texttt{FullDisc}) is the share of the test set with that set size, i.e., what full disclosure would reveal, so the gap between a curve and the dashed line gives the instances of that size the policy withholds. }
    \label{fig:cp_size}
\end{figure}

\section{Extended Related Work}
\label[appendix]{app:related}

\paragraph{Learning to defer and its extensions.}~\citep{DBLP:conf/nips/MadrasPZ18} introduce Learning to Defer (LtD), casting human-AI collaboration as a mutually exclusive choice: for each instance, either the model predicts or the case is deferred to a human expert (for a broader overview of rejection and deferral, see \citet{DBLP:conf/aaai/RuggieriP25}). The framework has since been developed along several axes. On the theoretical side,~\citep{DBLP:conf/icml/MozannarS20} derive consistent surrogate losses for the joint model-plus-deferral objective,~\citep{DBLP:conf/icml/VermaN22} propose a one-vs-all formulation with improved calibration of the deferral decision, and~\citep{DBLP:conf/aistats/MozannarLWSDS23} study exact algorithms for learning who should predict. For further work developing LtD methods with theoretical guarantees, see \citet{okati2021triage, charusaie2024deferandfusion, DBLP:conf/nips/CaoM0W023, DBLP:conf/aistats/LiuCZF024, DBLP:conf/aaai/GaoY25, li2017applying, DBLP:conf/icml/MontreuilCNO25, montreuil2026optimalqueryallocationextractive, DBLP:journals/tmlr/FangN26}. Other theoretical extensions of the LtD approach involved considering multiple-expert settings \citep{DBLP:conf/aistats/VermaBN23, DBLP:conf/nips/MaoMM023, DBLP:journals/corr/abs-2504-12988, DBLP:conf/icml/MaoM025, DBLP:conf/aaai/ZhangNWDRC26, DBLP:journals/corr/abs-2602-17144}, multi-task settings \citep{pugnana2025ask, DBLP:conf/icml/MontreuilHCNO25}, and aspects such as rejection of unexplainable decisions \citep{DBLP:journals/corr/abs-2507-12900}.

On the empirical side,~\citep{DBLP:conf/iui/HemmerWSVVS23} examine how delegation affects task performance and satisfaction with real participants, \citep{DBLP:conf/aistats/PalombaPAR25} evaluate deferring systems through a causal lens, and \citep{DBLP:conf/aaai/BondiKSCBCPD22} investigate how communicating a deferral affects human accuracy. Other works have examined the impact class distribution on rejected instances \citep{DBLP:journals/dmlr/PugnanaPDR24}, with \citet{pesenti2026samealgorithmichumanbias} examining how class imbalances may affect users' interactions with LtD systems. Deferral and rejection have also been studied in healthcare \citep{DBLP:conf/aaai/StrongMN25,DBLP:journals/npjdm/KompaSB21}, vehicle engineering \citep{DBLP:conf/adma/HendrickxMCD22}, and question answering \citep{montreuil2026optimalqueryallocationextractive}.
The defining feature of LtD is that authority over the final decision is allocated: on deferred instances the model abstains, and on the remaining ones the human is bypassed entirely. Our setting is structurally different. The human is \emph{always} the decision-maker, and what is allocated is not authority but \emph{information}: the policy decides what the human sees, never what the human decides. Consequently our risk is defined over human actions in two information regimes (\cref{eq:risks}), rather than over a model prediction and a human prediction.

\paragraph{AI-assisted decision-making.}
A second line of work keeps the human as the final decision-maker and asks what form the support should take.
Prediction sets are a prominent example: ~\citep{DBLP:conf/icml/StraitouriR24} design decision-support systems based on counterfactual prediction sets, ~\citep{DBLP:conf/nips/ToniOTSR24} construct prediction sets that target the expert's accuracy rather than coverage, and ~\citep{DBLP:conf/icml/CresswellSKV24} show, in a pre-registered randomized controlled trial, that conformal prediction sets improve human accuracy over fixed-size (top-$k$) sets with the same coverage.
~\citep{DBLP:conf/iui/SchemmerKBBS23} instead conceptualize appropriate reliance on AI advice and study how explanations affect it.
Broader syntheses are provided by ~\citep{DBLP:journals/corr/abs-2402-06287}, who survey learning paradigms for hybrid decision-making systems, and by ~\citep{DBLP:journals/ejis/HemmerSKVS25}, who conceptualize human-AI complementarity and review its sources and the empirical evidence for it.

Closer to our setting, a smaller body of work decides, instance by instance, whether (and which) support to show.
~\citep{DBLP:conf/ijcai/Noti023} study a regression task (pretrial risk assessment) and learn from past human predictions a policy that, given the case, the algorithmic risk score and the human's initial estimate, reveals the score only when it is predicted to be more accurate than that estimate; in a large-scale experiment, this improves human predictions over always showing the score.
~\citep{10.1145/3544548.3581058} do not learn the disclosure rule itself: they estimate the decision-maker's correctness likelihood on each instance by applying an approximation of their individual decision rules to similar labelled cases, and compare it with the AI's calibrated confidence; whenever the human is predicted to be more likely correct, the AI recommendation is either withheld (only its explanation is shown) or revealed only after an independent human judgement.
Other works choose among several forms of support.
~\citep{DBLP:conf/aaai/BhattCCKKWT25} cast this choice as a stochastic contextual bandit and learn online, separately for each new decision-maker, a policy that selects, for each input, the form of support (e.g., none, a model or LLM prediction, expert consensus) expected to minimize that individual's error.
~\citep{DBLP:journals/corr/abs-2403-05911} instead apply offline reinforcement learning to previously collected interaction data to choose among no assistance, explanation only, recommendation with explanation, and on-demand advice, optimizing for immediate accuracy, for the decision-maker's learning, or for both; their policies condition on a discrete state that includes the decision-maker's need for cognition and task knowledge and, as a proxy for AI uncertainty, the ground-truth correctness of the AI recommendation.
In the symmetric direction, ~\citep{pugnana2025ask} propose Learning to Ask, where a model decides under a budget when to query a human for enriched feedback, and characterize the optimal querying rule as a threshold on the risk difference between a standard and an enriched predictor; relatedly, ~\citep{wilder2020complement} and ~\citep{charusaie2024deferandfusion} train predictors that complement, or directly incorporate, human decisions.

Compared with the approaches that adapt the support shown to the human, LSD differs in three respects.
First, disclosure is subject to a budget, which makes the optimal policy a non-negative, budget-dependent threshold on VoI; without a budget and with only two actions, the optimal policy of ~\citep{DBLP:conf/aaai/BhattCCKKWT25} reduces to disclosing whenever an individual-level VoI is positive, i.e., the unconstrained ($B=1$) case of our threshold rule.
Second, VoI is the causal reduction in human decision risk induced by disclosure, and therefore depends on how decision-makers actually use the disclosed information, whereas ~\citep{DBLP:conf/ijcai/Noti023} and ~\citep{10.1145/3544548.3581058} condition disclosure on whether the AI is predicted to be more accurate than the human alone.
Third, we establish when VoI is identifiable from data collected under the two regimes, and bound both the degradation of the resulting plug-in policies relative to no disclosure and their regret relative to the optimal policy.

\paragraph{Active feature acquisition.}
In Active Feature Acquisition (AFA), an agent decides which missing feature values to acquire, at a cost, in order to improve a downstream predictive model~\citep{DBLP:journals/ml/JanischPL20,DBLP:journals/corr/abs-2502-11067,saar2009active,ji2007cost}.
~\citep{saar2009active} acquire feature values at training time to improve model induction, whereas at test time ~\citep{ji2007cost} formulate cost-sensitive acquisition and classification jointly and ~\citep{DBLP:journals/ml/JanischPL20} cast classification with costly features as a sequential decision problem solved by reinforcement learning.
~\citep{DBLP:journals/corr/abs-2502-11067} provide a unified view, organizing existing methods into embedded cost-aware predictors, model-based approaches, model-free policies, and hybrid strategies.
Related in spirit are ~\citep{DBLP:conf/icml/MaTPHNZ19} and ~\citep{DBLP:conf/nips/GongTNTHZ19}, who score candidate acquisitions with information-theoretic criteria.
Our VoI instead follows the classical decision-theoretic notion of value of information as the expected reduction in decision loss~\citep{4082064}, with one twist: for a Bayesian decision-maker this quantity is never negative, whereas ours is measured on actual human decisions and can be.
Two differences matter.
First, the acquisition target is a \emph{model}'s prediction rather than a \emph{human}'s decision, so the relevant risk is a model risk and the counterfactual ``what would the predictor do without this feature'' is directly computable, whereas the human counterfactual is not; this is precisely what forces the causal treatment of \cref{app:causal}.
Second, AFA chooses \emph{which} features to acquire, often sequentially and under a per-instance budget, whereas LSD makes a single binary decision on an indivisible block of support information, under a budget on the expected disclosure rate across cases.

\paragraph{Policy learning under budget constraints.}
The causal reading of our framework connects it to \emph{optimal policy learning} (OPL), which studies how to learn, from data, treatment assignment rules that maximize a welfare objective~\citep{https://doi.org/10.1111/j.1468-0262.2004.00530.x,https://doi.org/10.3982/ECTA15732,DBLP:journals/ijdsa/Cerulli26}.
Without constraints, the first-best rule treats every unit whose conditional average treatment effect (CATE) is positive; ~\citep{kitagawa2018should} propose empirical welfare maximization over constrained classes of treatment rules, ~\citep{BHATTACHARYA2012168} show that when a budget caps the fraction of treated units the optimal rule treats those whose CATE exceeds a quantile threshold, and ~\citep{DBLP:conf/aistats/CarranzaA25} extend policy learning to observational data from multiple sources using doubly robust estimators.
Our \cref{thm:optimal_disclosure} recovers a threshold rule of the same shape, with the VoI in the role of the CATE and the budget $B$ in the role of the capacity constraint.
As in budget-constrained OPL, the threshold is non-negative, so disclosure is withheld whenever it is expected to harm the decision, even if budget remains; in our setting this case is far from marginal, since support information can mislead the decision-maker, e.g., by inducing over-reliance.
The substantive difference is where the treatment acts.
In standard OPL the treatment is applied to the unit whose outcome is measured.
Here the treatment (disclosure) is applied to the \emph{decision-maker}, while the outcome is the correctness of a decision about a \emph{case}.
This has two consequences we make explicit in \cref{app:causal}: the missing-potential-outcome problem arises at the level of the decision episode rather than the case, and is addressed by assigning different decision-makers to the two regimes, which requires those observed under each regime to be exchangeable; and SUTVA must be stated over episodes, since what must not interfere is the information disclosed to a given decision-maker in a given episode.

\paragraph{Heterogeneous treatment effect estimation.}
Since the VoI is the negative conditional average treatment effect (CATE) of disclosure on decision loss (\cref{subsec:causal}), estimating it connects our work to the literature on heterogeneous treatment effects~\citep{nogueira2022methods}. Meta-learners reduce CATE estimation to standard supervised learning, either by fitting one outcome model per treatment arm (T-learner) or a single model with the treatment as input (S-learner)~\citep{kunzel2019metalearners}; more refined estimators include the R-learner~\citep{10.1093/biomet/asaa076}, doubly robust learners~\citep{10.1214/23-EJS2157}, causal forests~\citep{Wager03072018} and neural architectures sharing representations across arms~\citep{DBLP:conf/icml/ShalitJS17,DBLP:conf/nips/ShiBV19}. In our setting, doubly robust corrections are less critical than in observational studies, since each case is observed under both regimes and disclosure is assigned independently of the case covariates. 
Moreover, accurate effect estimation is neither necessary nor sufficient for good causal decisions~\citep{doi:10.1287/ijds.2021.0006}: ~\citep{DBLP:conf/nips/FrauenMSSF25} show that thresholding CATE estimators trained for estimation accuracy can yield suboptimal policies, since accuracy away from the decision boundary is irrelevant to the decision, and ~\citep{DBLP:journals/corr/abs-2602-03517} show that, when individuals must be prioritized, recovering the ranking of treatment effects is easier than estimating their magnitude, and propose to learn it directly. Our setting requires both: under a binding budget, the optimal policy of \cref{thm:optimal_disclosure} depends only on the ranking of the VoI, while the non-negative threshold depends on its sign. These are precisely the properties on which our plug-in guarantees (\cref{subsec:guarantees}) and our ablation (\cref{app:q1}) focus.

\end{document}